\documentclass[]{article}
\usepackage{url}

\usepackage{amsmath}
\usepackage{amsfonts}
\usepackage{amsthm}
\usepackage{amssymb}
\usepackage{bbm}
\usepackage{graphicx,psfrag,epsf}
\usepackage{enumerate}
\usepackage{natbib}
\usepackage{color}
\usepackage{booktabs}
\usepackage{algorithm,algorithmicx,algpseudocode}
\usepackage{url}
\usepackage{pdflscape}

\usepackage[margin=1in]{geometry}

\newtheorem{theorem}{Theorem}
\numberwithin{theorem}{section}
\newtheorem{assumption}{Assumption}
\numberwithin{assumption}{section}
\newtheorem{example}{Example}
\newtheorem{lemma}{Lemma}
\numberwithin{lemma}{section}

\newcommand{\known}{\text{known}}
\newcommand{\plugin}{\text{plugin}}

\newcommand{\rel}{\text{rel}}
\newcommand{\bayes}{\text{bayes}}
\newcommand{\LR}{\text{LR}}

\DeclareMathOperator{\E}{\mathbb{E}}

\begin{document}

\title{Monitoring machine learning (ML)-based risk prediction algorithms in the presence of confounding medical interventions}

\author{
	Jean Feng$^{1}$, Alexej Gossmann$^{2}$, Gene Pennello$^{2}$,
Nicholas Petrick$^{2}$,\\ Berkman Sahiner$^{2}$, Romain Pirracchio$^{1}$
}
\date{
$^1$University of California, San Francisco\\
$^2$U.S. Food and Drug Administration, Center for Devices and Radiological Health
}

\maketitle

\begin{abstract}
	Performance monitoring of machine learning (ML)-based risk prediction models in healthcare is complicated by the issue of confounding medical interventions (CMI):
	when an algorithm predicts a patient to be at high risk for an adverse event, clinicians are more likely to administer prophylactic treatment and alter the very target that the algorithm aims to predict.
	A simple approach is to ignore CMI and monitor only the untreated patients, whose outcomes remain unaltered.
	In general, ignoring CMI may inflate Type I error because (i) untreated patients disproportionally represent those with low predicted risk and (ii) evolution in both the model and clinician trust in the model can induce complex dependencies that violate standard assumptions.
	Nevertheless, we show that valid inference is still possible if one monitors conditional performance and if either conditional exchangeability or time-constant selection bias hold.
	Specifically, we develop a new score-based cumulative sum (CUSUM) monitoring procedure with dynamic control limits.
	Through simulations, we demonstrate the benefits of combining model updating with monitoring and investigate how over-trust in a prediction model may delay detection of performance deterioration.
	Finally, we illustrate how these monitoring methods can be used to detect calibration decay of an ML-based risk calculator for postoperative nausea and vomiting during the COVID-19 pandemic.
\end{abstract}

\maketitle

\section{Introduction}

After a machine learning (ML)-based system is deployed in clinical practice, real-world monitoring of the algorithm for potential performance degradation is necessary for mitigating risk and is an important aspect of good machine learning practice (GMLP) \citep{US_Food_and_Drug_Administration2021-tz}.
Various methods for performance monitoring are available \citep{Feng2022-mk, Kahn1996-yx}, such as those from statistical process control.
All of these procedures assume an ideal data setting in which the prediction target is observed.
However, the data available for monitoring a ML-based risk prediction algorithm are often subject to confounding medical interventions (CMI), because predictions from the algorithm can alter the very outcome that it aims to predict \citep{Paxton2013-pv, Lenert2019-uf, Perdomo2020-cz}.

As a motivating example, consider the Targeted Real-time Early Warning System (TREWS) sepsis risk prediction algorithm \citep{Adams2022-rm}, which estimates the probability of a patient developing septic shock if they only receive standard of care (SOC) and no additional interventions.
This algorithm was recently shown to reduce in-hospital mortality and organ failure rates.
The alert system increased the propensity and speed at which clinicians submitted their first antibiotic order \citep{Henry2022-pg}.
Moreover, the authors found that clinicians' likelihood to interact with the TREWS alert system depended on their previous interactions, and hypothesize that clinician trust will continue to evolve with increased exposure to ML-based systems.
Monitoring the performance of the TREWS algorithm is especially important because (i) it depends on electronic health record (EHR) data, which is prone to distribution shifts, and (ii) overuse of antibiotics has major negative consequences.
To evaluate the model's predictions for patient outcomes under SOC, a potential solution is to restrict our attention to only those patients who received SOC, as the counterfactual outcomes for the patients who received antibiotics are unknown.
However, in the likely scenario where high-risk patients are preferentially selected to receive an intervention, procedures that monitor marginal performance measures (e.g., misclassification rate or AUC) without adjusting for CMI are biased, due to dependent censoring of the outcome (Figure~\ref{fig:naive} of the Appendix).
This can lead to inflated false alarm rates and/or unnecessarily long detection delays.

In the offline setting, one can try to address the mismatch between the ``SOC-only'' and general target population by reweighting the data by the inverse of their treatment propensities.
In the sequential setting, the closest work along these lines is \citet{Sun2014-yh}, which combines inverse censoring weights with a cumulative sum (CUSUM) algorithm to monitor survival outcomes.
However, proper error rate control is contingent on knowing the exact weights, which is unlikely to hold in our setting.
Moreover, we may not even be able to accurately estimate the propensity weights, because it is difficult to anticipate how clinician trust in an algorithm will evolve over time.
Finally, even with access to exact weights, near-violations of the positivity condition (i.e. weights close to zero) can drastically slow down asymptotic convergence of these procedures, thereby inflating the false alarm rate.

Given the difficulties of monitoring \textit{marginal} performance measures using a propensity-based approach, we propose to monitor \textit{conditional} measures of performance, which we show are independent of treatment propensities under one of two ignorability conditions.
The first is the well-known conditional exchangeability assumption \citep{Rubin1976-of}.
The second is the assumption of time-constant selection bias, which has not yet been discussed in the literature to our knowledge.
If either condition holds, we can ignore CMI, use a ``standard'' monitoring procedure to analyze the SOC-only data, and avoid estimating treatment propensities altogether.
We pay special attention to the model calibration, a popular conditional performance measure and one of the most common types of performance deterioration in practice \citep{Hickey2013-qm, Davis2017-bl}.

In addition, we address two challenges in analyzing performance of ML-based risk prediction models.
First, the sequence of predictors among the SOC-only patients can be highly nonstationary, because the ML algorithm and/or the clinician's interactions with the ML algorithm can evolve over time.
Second, the exact performance characteristics of an ML algorithm are often unknown upfront and must be estimated, whereas many monitoring algorithms were originally designed for settings where the pre-change data distribution is known exactly (e.g. industrial manufacturing).
Although there are procedures that partially address these challenges \citep{Dette2020-mr, Zeileis2007-la, Gombay2017-wb}, we are not aware of a frequentist monitoring procedure that adequately addresses both.
To this end, we introduce a new nonanticipating score-based CUSUM chart statistic and a computationally efficient procedure for generating dynamic control limits.
We note that Bayesian approaches naturally handle nonstationarity and sources of uncertainty, but posterior inference is computationally challenging over long periods of time and for complex models of performance decay \citep{Shiryaev1963-xq, West1986-jt, Bhattacharya1994-cp}.
As such, we leave the study of Bayesian monitoring methods to future work.

Through simulation studies, we explore how various factors such as clinician trust and model retraining impact our ability to detect performance decay.
We also apply score-based CUSUM monitoring to detect calibration decay of a ML-based risk calculator for postoperative nausea and vomiting (PONV) on data from the Multicenter Perioperative Outcomes Group (MPOG).
For a locked PONV risk calculator, the procedure fires an alarm during the COVID-19 pandemic.
In contrast, when we continually retrain the model, we find that the model can steadily adapt to temporal shifts while remaining well-calibrated.
Code is available at \url{https://github.com/jjfeng/monitoring_ML_CMI}.

\section{Two monitoring problems}
Here we introduce the sequential monitoring problem in the ``standard'' setting and then extend this to the setting with CMI.
In the following section, we discuss conditions under which a monitoring procedure designed for the former setting provides valid inference for the latter.
For convenience, Table~\ref{tab:symbols}  summarizes the mathematical notation used in this paper.

\subsection{The standard monitoring problem}
Following the framework set forth in \citet{Chu1996-fg}, suppose one observes covariates $Z_t$ and predictors $Y_t$ for times $t = 1,\cdots, m$, during which the conditional distribution $Y_t | Z_t$ is constant.
For some integer $K > 1$, we are interested in detecting structural change in this conditional distribution for sequentially arriving observations $(Z_t, Y_t)$ for times $t = m+1,\cdots, mK$.
We assume there is at most one changepoint $\kappa$ during this monitoring period, where $\kappa = \lfloor m \kappa^{\rel}\rfloor $ for some $\kappa^{\rel} \in (1,K)$.
If such a changepoint exists, the conditional distribution $Y_t|Z_t$ is assumed to follow some model with parameters $(\theta, \delta \mathbbm{1}\{t \ge \kappa\})$, where $\theta \in \mathbb{R}^p$ describes the pre-change distribution and $\delta  \in \mathbb{R}^d$ describes the shift.
The outcome $Y_t$ is assumed to be conditionally independent of historical data given $Z_t$, so that the distribution can be factorized into
\begin{align}
	\hspace{-0.2in}
	\begin{split}
		\Pr\left(Y_1,\cdots, Y_{t}, Z_1,\cdots, Z_{t} \right)=
		& \prod_{i=1}^{t} \Pr\left(Y_i \mid Z_i ; \theta, \delta\mathbbm{1}\{i \ge \kappa \} \right) \Pr\left(Z_i \mid Z_1, Y_1, \cdots, Z_{i-1}, Y_{i-1}; \eta \right)
		\label{eq:monitoring_base}
	\end{split}
\end{align}
for $t = 1,\cdots, mK$, where $\eta$ describes the conditional distribution of $Z_i$ with respect to past monitoring data.
The hypothesis test of interest is
\begin{align}
	\begin{split}
		H_0 &: Y_i | Z_{i} \sim (\theta, 0) \quad \forall i=1,\cdots, mK\\
		H_1&: \exists \delta, \kappa \text{ s.t. } Y_i| Z_i \sim (\theta, \delta \mathbbm{1}\{i \ge \kappa \rfloor \}) \quad \forall i=1,\cdots, mK,
		\label{eq:main_hypo}
	\end{split}
\end{align}
where $\theta$ is a nuisance parameter.

We define a monitoring procedure as one with chart statistic $C_{m}(t)$ and dynamic control limit (DCL) $h_{m}(t)$ at times $t = m+1, \cdots, mK$.
The procedure fires an alarm when the chart statistic first exceeds the control limit, i.e.
$
\hat{T}_m = \inf\left\{
t: C_{m}(t) > h_{m}(t)
\right\}.$
Assuming the operating characteristics for a monitoring procedure have been established for the standard setting, our interest is to understand when these properties transfer to the CMI setting.

\subsection{The CMI monitoring problem}
Data are generated in the CMI setting as follows.
Upon arrival of a patient with covariates $(X_t, \tilde{X}_t)$ at time $t$, the ML algorithm $\hat{f}_t:\mathcal{X} \mapsto \mathcal{Q}$ outputs prediction $\hat{f}_t(X_t)$.
The clinician then takes into account various factors---such as the prediction, patient covariates, and prior experience with the algorithm---to decide treatment $A_t$, where $A_t = 0$ indicates SOC and $A_t = 1$ indicates additional intervention.
Finally patient outcome $Y_t$ is observed.
Following the potential outcomes framework, let $Y_{t}(a)$ indicate the patient outcome if treatment $a$ is administered.
We assume the realized outcome $Y_t$ is equal to the potential outcome $Y_t(A_t)$ (consistency).
Note that $t$ denotes patient ordering, not repeated observations per patient.

This paper will focus on monitoring ML algorithms that predict the outcome under SOC $Y_t(0)$, though our results can be extended to monitor algorithms that predict treatment-specific outcomes.
The central question is whether we can simply monitor the subsequence of patients who were assigned SOC, as their outcomes remain unaltered.
To this end, let $\tau_i$ be the time of the $i$th patient who received SOC, which can be viewed as a random stopping time.
We refer to the data observed at time points $\{\tau_i: i = 1,\cdots, mK\}$ as SOC-only data.

In simple cases, we aim to detect structural change in $Y_{t}(0) \mid \hat{f}_{t}(X_{t})$.
This corresponds to monitoring the calibration curve if $\hat{f}_{t}$ is risk prediction algorithm and positive/negative predictive values if $\hat{f}_t$ is a binary classifier.
For static algorithms, a structural change can only occur if there is a shift in the joint distribution $(X_t, Y_t(0))$.
For evolving ML algorithms, a structural change can occur either due to a shift in $(X_t, Y_t(0))$ and/or the updating procedure fails to maintain this conditional relationship.
Both failure modes are reasons for alarm.

When we directly apply a standard monitoring procedure to SOC-only data, this corresponds to replacing $Z_i$ with $\hat{f}_{\tau_i}(X_{\tau_i})$ and $Y_i $ with $Y_{\tau_i}$ to test the hypothesis
\begin{align}
	\begin{split}
		H_0 &: Y_{t}(0) | \hat{f}_{t}(X_{t}), t = \tau_i \sim (\theta, 0) \quad \forall i=1,\cdots, mK\\
		H_1&: \exists \delta, \kappa \text{ s.t. } Y_{t}(0) | \hat{f}_{t}(X_{t}), t = \tau_i \sim (\theta, \delta \mathbbm{1}\{i \ge \kappa \}) \quad \forall i=1,\cdots, mK.
	\end{split}
	\label{eq:cmi_hypothesis}
\end{align}
Note that this assumes the changepoint corresponds to the random stopping time $\tau_{ \lfloor m\kappa^{\rel} \rfloor }$.
In more complex settings where clinical decision-making also depends on $\tilde{X}_t$, it is useful to monitor for shifts in $Y_t(0)|\hat{f}_t(X_t), \tilde{X}_t$.
This corresponds to replacing $Z_i$ with $\left( \hat{f}_{\tau_i}(X_{\tau_i}), \tilde{X}_{\tau_i}\right)$ and $Y_i$ with $Y_{\tau_i}$ in the standard monitoring procedure to test
\begin{align}
	\begin{split}
		H_0&: Y_{t}(0) | \hat{f}_{t}(X_{t}), \tilde{X}_{t}, t = \tau_i \sim (\theta, 0) \quad \forall i=1,\cdots, mK\\
		H_1&: \exists \delta, \kappa \text{ s.t. } Y_{t}(0) | \hat{f}_{t}(X_{t}),\tilde{X}_{t},  t=\tau_i \sim (\theta, \delta \mathbbm{1}\{i \ge \kappa \}) \quad \forall i=1,\cdots, mK
	\end{split}.
	\label{eq:cmi_hypothesis_extend}
\end{align}
One can view \eqref{eq:cmi_hypothesis_extend}  as monitoring stratified calibration curves---a stricter notion of calibration in the hierarchy defined by \citet{Van_Calster2016-ey}---or stratified predictive values.

In general, the operating characteristics of a monitoring procedure designed for the standard setting may not transfer to the CMI setting because (i) the distribution of the SOC-only data may not factorize per \eqref{eq:monitoring_base} and (ii) the parameters may be biased.
Nevertheless, the following section highlights ignorability conditions under which valid statistical inference \textit{does} transfer.
We note that no positivity assumptions are required to establish control of the false alarm rate, because we only monitor for structural change with respect to the population with non-zero probability of receiving SOC, even though the null hypothesis applies to the broader population.
Nevertheless, statistical power and consistency of the monitoring procedure does depend on positivity, as we show in our numerical experiments.

\section{Ignorability assumptions}
\label{sec:ignore}

We now describe two conditions in which standard monitoring procedures applied directly to SOC-only data, ignoring issues of CMI, provide valid inference.
The first ignorability condition is the well-known conditional exchangeability assumption, also known as no unmeasured confounding.
The second ignorability condition allows treatment decisions to depend on unmeasured confounders but requires selection bias to remain constant over time.

\subsection{Conditional exchangeability}

The simplest version of the conditional exchangeability assumption states that treatment assignment $A_t$ is conditionally independent of potential outcome $Y_t(0)$ given prediction $\hat{f}_t(X_t)$, i.e.
\begin{align}
	Y_t(0) &\perp A_t \mid \hat{f}_t(X_t)
	\quad \forall t = 1,2,\cdots.
	\label{eq:exchangeable}
\end{align}
In addition, let us assume the potential outcome $Y_t(0)$ is conditionally independent of data at prior time points per
\begin{align}
	\hspace{-0.3in}
	Y_{t}(0)
	&\perp \left(A_1, \hat{f}_1(X_1), Y_1(0), \cdots, A_{t-1}, \hat{f}_{t-1}(X_{t-1}), Y_{t-1}(0)\right)
	\mid\hat{f_{t}}(X_t), A_t = 0 \quad \forall t = 1,2,\cdots.
	\label{eq:indpt}
\end{align}
These conditions hold, for instance, if the ML algorithm is static (i.e. $\hat{f}_t(X_t) \equiv \hat{f}(X_t)$ for all $t$) and the treatment propensity at time $t$ solely depends on the ML algorithm's predictions.
Figure~\ref{fig:dags} top shows a more complex example, in which both the treatment propensities and ML algorithm evolve.
To ensure that the conditional independence assumption \eqref{eq:indpt} holds, there cannot be an arrow from $Y_{t-1}(a_{t-1})$ to $\hat{f}_{t}(X_t)$, because this would introduce collider bias.
As such, we must keep the data for model updating and monitoring separate, through mechanisms such as online sample-splitting.

\begin{figure}
	\centering
	\includegraphics[width=0.5\linewidth]{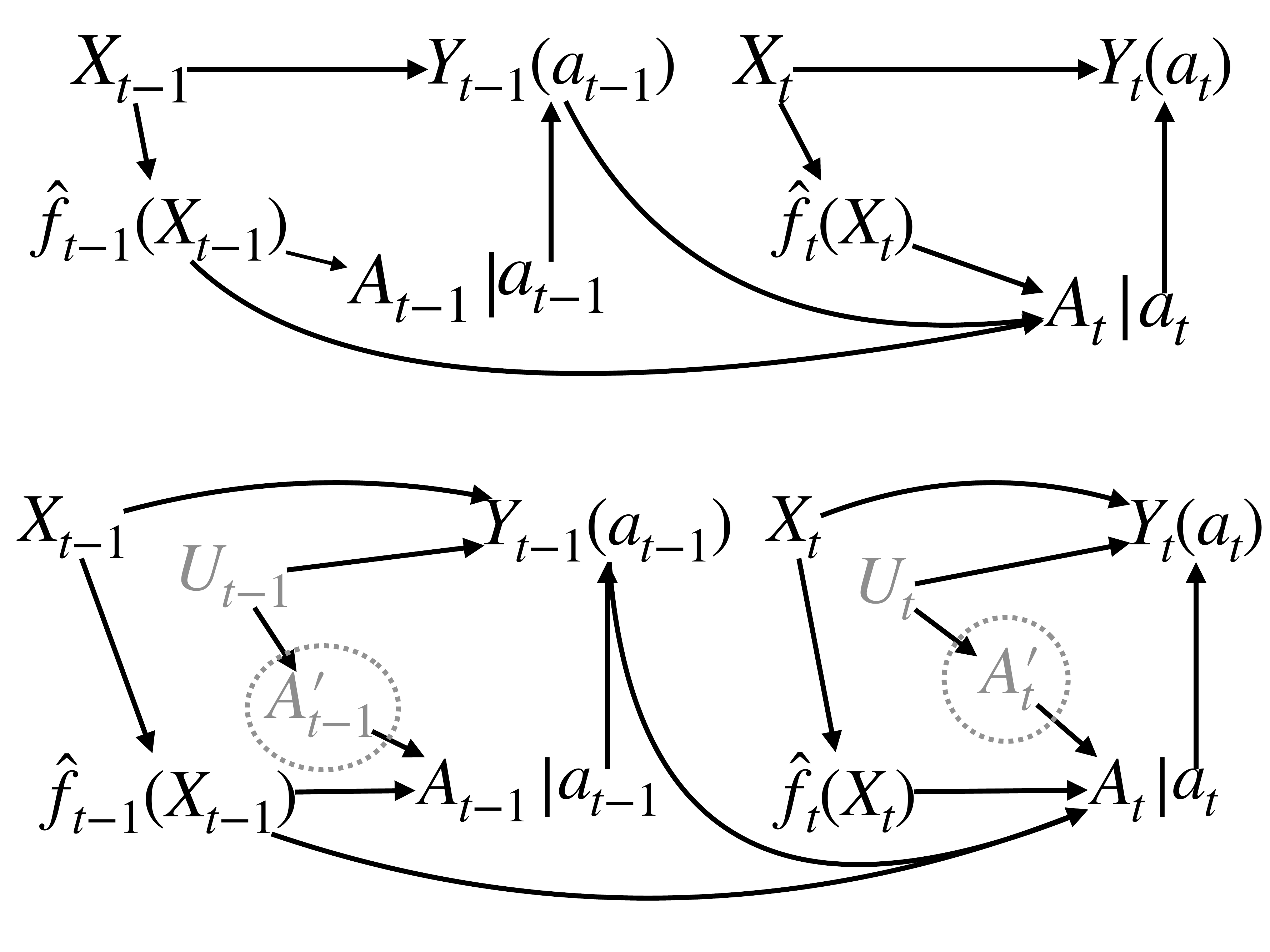}
	\caption{
		The top single world intervention graph (SWIG) \citep{Richardson2013-ol} is an example satisfying conditional exchangeability \eqref{eq:exchangeable} and conditional independence \eqref{eq:indpt}.
		Black and gray variables are observed and unobserved, respectively.
		The bottom SWIG satisfies time-constant selection bias \eqref{eq:equiconf} and conditional independence \eqref{eq:indpt} per the assumptions listed in Example~\ref{example:equiconf}.
		The dotted ovals indicate determinative causes where $A_t = 1$ whenever $A_{t}' = 1$ for all $t$.
	}
	\label{fig:dags}
\end{figure}

Given \eqref{eq:indpt}, the distribution of SOC-only data can be factorized into
\begin{align}
	\begin{split}
		&\Pr\left(Y_{\tau_1}, \cdots, Y_{\tau_{t}}, \hat{f}_{\tau_1}(X_{\tau_1}), \cdots, \hat{f}_{\tau_{t}}(X_{\tau_{t}})\right)\\
		= & \prod_{i=1}^{t} \Pr\left(Y_{\tau_i} | \hat{f}_{\tau_i}(X_{\tau_i}) \right)
		\Pr\left(
		\hat{f}_{\tau_i}(X_{\tau_i}) \mid \hat{f}_{\tau_1}(X_{\tau_1}), \cdots, \hat{f}_{\tau_{i-1}}(X_{\tau_{i-1}}),
		Y_{\tau_1}, \cdots, Y_{\tau_{i-1}}
		; \eta
		\right).
		\label{eq:cmi_indpendence}
	\end{split}
\end{align}
Then per the conditional exchangeability assumption \eqref{eq:exchangeable}, we have that the conditional distribution in the SOC-only data coincides with that of the general population, i.e.
\begin{align}
	\Pr(Y_{\tau_i} | \hat{f}_{\tau_i}(X_{\tau_i}))
	= & \Pr(Y_{t}(0) | \hat{f}_{t}(X_{t}), t=\tau_i)
	,
	\label{eq:cmi_factors}
\end{align}
which is parameterized by $(\theta, \delta\mathbbm{1}\{i \ge \kappa\})$.
So the distribution of SOC-only data factorizes into the same form as \eqref{eq:monitoring_base} with monitoring target $Y_{t}(0) | \hat{f}_{t}(X_{t}), t=\tau_i$ and a standard monitoring procedure that ignores CMI would provide valid inference for \eqref{eq:cmi_hypothesis}.

As is typically the case, the propensity to treat may also depend on other predictors $\tilde{X}_t$.
We can extend \eqref{eq:exchangeable} and \eqref{eq:indpt} by conditioning on $\tilde{X}_t$ as well.
Using similar arguments as above, we can show that ignoring CMI still provides valid inference for testing \eqref{eq:cmi_hypothesis_extend}.

\subsection{Time-constant selection bias}
\label{sec:exch}

When unmeasured confounders $U_t$ exist, conditional exchangeability no longer holds and conditioning on $A_t =0 $ results in selection bias (Figure~\ref{fig:dags} bottom).
That is, the observed relationship between the outcome and the ML predictions in the SOC-only data does not necessarily correspond to that of the general population.
Nevertheless, we can show that the treatment propensities are still ignorable if selection bias remains constant over time.

The simplest version of the time-constant selection bias assumption states that there is some function $h: \mathcal{Q}\mapsto\mathbb{R}$ such that
\begin{align}
	\hspace{-0.2in}
	E\left[Y_{t}(0)\mid \hat f_{t}(X_t) = q\right]-E\left[Y_{t}(0)\mid \hat f_{t}(X_t)= q,A_t=0\right]=h\left(q\right) \qquad \forall q \in \mathcal{Q}, \forall t= 1,2,\cdots.
	\label{eq:equiconf}
\end{align}
Nevertheless, we can still recover shifts in the conditional risk in the general population from the SOC-only data because this bias cancels out, i.e.
\begin{align}
	\begin{split}
		& E\left[Y_{t}(0)\mid \hat f_{t}(X_t)=q,A_t=0\right]-E\left[Y_{1}(0)\mid \hat f_{1}(X_{1})=q,A_{1}=0\right]\\
		=&
		E\left[Y_{t}(0)\mid\hat{f}_{t}(X_{t})=q\right]-E\left[Y_{1}(0)\mid\hat{f}_{1}(X_{1})=q\right]
		\qquad \forall q \in \mathcal{Q}, \forall t= 1,2,\cdots.
		\label{eq:equiconf_selection}
	\end{split}
\end{align}
So combining time-constant selection bias with the conditional independence assumption \eqref{eq:indpt}, we can plug into the factorization \eqref{eq:cmi_indpendence}
\begin{align}
	\Pr(Y_{\tau_i} | \hat{f}_{\tau_i}(X_{\tau_i}) = q) =
	\Pr(Y_{t}(0)| \hat{f}_{t}(X_{t}) = q, A_{t} = 0, t = \tau_i; \theta', \delta\mathbbm{1}\{i \ge \kappa\})
\end{align}
where $\theta'$ describes the conditional distribution  $Y_t(0)|\hat{f}_t(X_t), A = 0$ prior to the changepoint.
Importantly, the value of $\delta$ remains unbiased, even though $\theta'$ may not equal $\theta$.
Since $\theta'$ is treated as a nuisance parameter, a standard monitoring procedure that ignores CMI can still provide valid inference for testing \eqref{eq:cmi_hypothesis}.
A straightforward extension is to condition on additional variables $\tilde{X}_t$ in assumptions \eqref{eq:indpt} and \eqref{eq:equiconf}, which implies that a standard monitoring procedure that ignores CMI provides valid inference for testing \eqref{eq:cmi_hypothesis_extend}.

When does time-constant selection bias hold?
We discuss one set of conditions in Example~\ref{example:equiconf} of the Appendix for the DAG (Figure~\ref{fig:dags} bottom).
Briefly, suppose the clinician makes an initial treatment decision $A_{t}' \in \{0,1\}$ that she modifies based on some unmeasured confounder $U_t$ (e.g. a biomarker)
One of the conditions is that $A_t = 1$ whenever $A_{t}' = 1$, also known as determinative causation \citep{Hernan2004-ll, VanderWeele2007-ib, VanderWeele2009-pa}.

\section{The score-based CUSUM}
\label{sec:score_monitor}
In this section, we describe a new score-based CUSUM procedure with DCLs to detect performance deterioration in the CMI setting.
An example control chart for this procedure is shown in Figure~\ref{fig:example_charts}.
The procedure is specially designed to address two key challenges that arise when monitoring ML-based risk prediction algorithms.
First, the predictor sequence may be nonstationary due to changes in the model and/or the clinician's interactions with the model over time.
Second, the initial performance of the algorithm may be unknown and must be estimated.
DCLs have been previously used to address nonstationarity, but only for likelihood-based test statistics and in settings where the pre-change and shift parameters are known \citep{Zhang2017-iw}, as the distribution of such test statistics can be calculated under such assumptions.
Instead, we propose to use a \textit{nonanticipative} score-based chart statistic, in that the score for the $i$-th observation is calculated with respect to parameters estimated using only historical data \citep{Lorden2005-bw}.
This preserves the martingale structure of the chart statistic even when the nuisance parameter $\theta$ for the pre-change distribution is continually re-estimated.
We can then derive its asymptotic distribution and  efficiently construct DCLs using a simple parametric bootstrap procedure.

\begin{figure}
	\centering
	\includegraphics[width=0.4\linewidth]{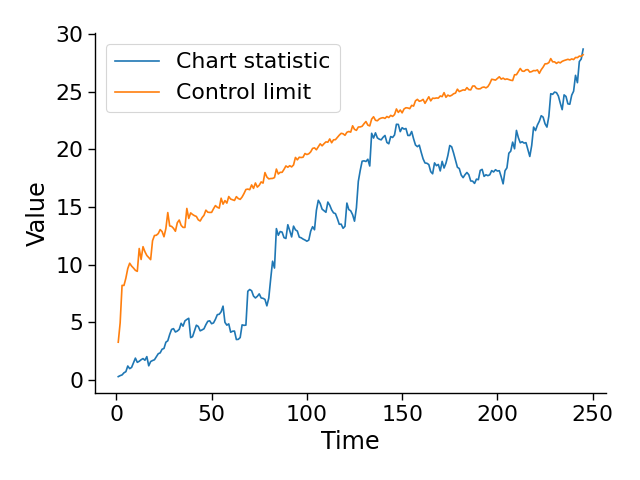}
	\caption{
		Example control chart for the score-based CUSUM.
		The chart statistic and control limits are shown in the blue and orange lines, respectively.
		An alarm is fired when the chart statistic exceeds the control limit.
	}
	\label{fig:example_charts}
\end{figure}

For ease of notation, we will present the monitoring procedures in the context of the standard setting.
The procedure can be applied directly to the CMI setting if the aforementioned ignorability conditions hold.

We quantify the operating characteristics of our procedure in terms of the false alarm rate and its statistical power.
That is, the probability of firing an alarm before the changepoint, $\Pr\left(
\hat{T}_m < \kappa
\right)$, should be controlled at level $\alpha > 0$.
In addition, the procedure should be consistent, in that
$\lim_{m \rightarrow\infty }
\Pr\left(
\hat{T}_m \le m K
\right) = 1$ under $H_1$.
All proofs are in the Appendix.

For the theoretical analyses, we suppose the parametric model for $Y_t|Z_t$ is correctly specified.
For concreteness, consider the following examples, which we will use in our empirical analyses.
The first model describes both the pre-change distribution and the structural change on the log odds scale using logistic regression, i.e.
\begin{align}
	\Pr\left(Y_t = 1 \mid Z_t ; \theta, \delta\mathbbm{1}\{t \ge \kappa\} \right) &=
	\frac{1}{1 + \exp\left (-(\theta + \delta\mathbbm{1}\{t \ge \kappa\})^\top Z_t \right)}.
	\label{eq:mdl_de}
\end{align}
The second model describes the pre-change distribution on the log odds scale but the structural change on the risk scale using
\begin{align}
	\Pr\left(Y_t = 1 \mid Z_t  ; \theta, \delta\mathbbm{1}\{t \ge \kappa\} \right) &=
	\left[
	\frac{1}{1 + \exp\left (-\theta^\top Z_t \right)} + (\delta \mathbbm{1}\{t \ge \kappa\})^\top Z_t
	\right]_{[0,1]},
	\label{eq:mdl_tc}
\end{align}
where $[x]_{[0,1]} = \min(1, \max(0, x))$.
When conditional exchangeability holds, we can monitor for structural change on any scale and use either \eqref{eq:mdl_de} or \eqref{eq:mdl_tc}.
If only time-constant selection bias holds, we are restricted to modeling shifts on the risk scale and may only use \eqref{eq:mdl_tc}.

\subsection{Known pre-change parameter}
\label{sec:prechange_known}
We begin with the simplest setting in which the true value of the nuisance parameter $\theta$, denoted $\theta_0$, is known.
This holds, for instance, when the model is known to be well-calibrated and one aims to test \eqref{eq:cmi_hypothesis}.
Note that the noncontamination dataset not needed in this setting.

For observation $(Z_t, Y_t)$, the score vector with respect to $\delta$ under the null hypothesis is $\left . \nabla_\delta \log p\left(Y_t \mid Z_t; \theta_0, \delta \right) \right|_{\delta = 0}$.
Because the conditional mean of the score is zero prior to the changepoint and nonzero after, we will monitor for shifts in the average score using the cumulative sum (CUSUM) \citep{Page1954-gr}.
In particular, for candidate changepoint $t'$, define the cumulative score up to time $t$ as
\begin{align}
	\psi^{(\known)}_m(t', t) = \sum_{i=t'}^t
	\left . \nabla_\delta \log p\left(Y_i \mid Z_i; \theta_0, \delta  \right) \right|_{\delta = 0}.
	\label{eq:psi_oracle}
\end{align}
Since the true changepoint time is unknown, the score-based CUSUM with respect to norm $\|\cdot\|$ is defined as
\begin{align}
	C_{m}^{(\known)}(t) =
	\max_{t' = m + 1,\cdots, t}
	\left \|
	\psi^{(\known)}_m(t', t)
	\right \|.
\end{align}
In our empirical analyses, we use $\|\cdot\|_1$, though one can consider other norms (for instance, using $\|\cdot\|_2$ would be more similar to Rao's score statistic).

We define DCLs $h_{m}(t)$ recursively using an alpha spending approach.
Let $\alpha^{\rel}: [1,K] \mapsto [0,1]$ be the alpha-spending function, where $\alpha^{\rel}$ is assumed to be continuous and monotonically non-decreasing.
Then $h_{m}(t)$ is the minimal threshold at which the conditional false alarm rate up to time $t$ matches the prespecified alpha-spending rate, i.e.
\begin{align}
	\hspace{-0.4in}
	\Pr\left(
	\exists t' \in \{m + 1,2,...,t\} \text{ such that }
	C_{m}^{(\known)}(t') > h_{m}(t')
	\mid \{Z_{t'}: t' = 1,...,t\}
	\right) \le \alpha^{\rel}(t/m)
	\label{eq:control_oracle}
\end{align}
under the null.
To calculate the DCLs, we resample outcomes $Y^{*}_t$ given $Z_t$ and the known pre-change parameters for all time points $t$, constructing sequences $\{ (Z_t, Y^{*(b)}_t) : t = m+1,\cdots, mK \}$ for $b = 1,\cdots, B$.
By choosing a sufficiently large value for $B$, we then estimate the distribution of the chart statistic and compute the DCLs.

We can show that the monitoring procedure is consistent if the average score after the changepoint is bounded away from zero under the alternative, i.e. there is some $c > 0$ and $K'\in(\kappa^{\rel}, K]$ such that
\begin{align}
	\lim_{m \rightarrow \infty}
	\left\|
	\frac{1}{m}
	\sum_{t=\lfloor m \kappa^{\rel} \rfloor }^{\lfloor mK'\rfloor }
	\mathbb{E} \left[
	\left . \nabla_{\delta} \log p \left(Y_t | Z_t; \theta_0, \delta \right) \right|_{\delta = 0}
	\right] \right\|
	\ge c,
	\label{eq:alt_bounded_away}
\end{align}
and the martingale
\begin{align}
	\sum_{t=\lfloor m \kappa^{\rel} \rfloor }^{\lfloor mK'\rfloor}
	\left . \nabla_{\delta} \log p \left(Y_t | Z_t; \theta_0, \delta \right) \right|_{\delta = 0}
	- \mathbb{E} \left[
	\left . \nabla_{\delta} \log p \left(Y_t^* | Z_t; \theta_0, \delta \right) \right|_{\delta = 0}
	\mid  Z_t
	\right]
\end{align}
satisfies the martingale central limit theorem.

\subsection{Unknown pre-change parameter}
\label{sec:score_unknown}
When the pre-change parameter $\theta$ is unknown, we need to estimate its value and adjust the DCLs to reflect this additional source of uncertainty.
This situation is likely to occur when calibration data is limited (e.g. deploying a model at a new site) or when monitoring the more complex conditional distribution $Y_t(0)| \hat{f}_t(X_t), \tilde{X}_t$.
To this end, let the maximum likelihood estimator (MLE) for $\theta$ up to time $t$, denoted $\hat{\theta}_{m,t}$, be the solution to the estimating equation
$\sum_{i=1}^{t}\nabla_{\theta}\log p\left(Y_{i}\mid Z_{i};\theta, 0\right) = 0$.
We define the score-based CUSUM chart statistic
\begin{align}
	C_{m}^{(\plugin)}(t) =
	\max_{t' = m + 1,\cdots, t}
	\left \|
	\psi^{(\plugin)}_m(t', t)
	\right \|,
	\label{eq:plugin_def}
\end{align}
where
\begin{align}
	\psi^{(\plugin)}_m(t', t)
	= \sum_{i=t'}^t
	\left . \nabla_\delta \log p\left(Y_i \mid Z_i; \hat{\theta}_{m,i-1}, \delta  \right) \right|_{\delta = 0}.
	\label{eq:plugin_sub}
\end{align}
The score for observation $(Z_i, Y_i)$ uses $\hat{\theta}_{m,i-1}$ rather  than $\hat{\theta}_{m,i}$, so \eqref{eq:plugin_sub} is nonanticipative.

To determine the operating characteristics of $C_m^{(\plugin)}$ under the null hypothesis, we require the following assumptions.
We use $\delta_0 = 0$ to denote the value of $\delta$ under the null.
\begin{assumption}
	Under the null, there is a zero-mean $(p+d)$-dimensional non-degenerate
	gaussian process $U$ such that
	\begin{equation*}
		\max_{m+1\le i\le mK}\left\Vert
		\left[
		\frac{1}{\sqrt{m}}\sum_{j=1}^{i}\left(\begin{array}{c}
			\nabla_{\theta}\log p\left(Y_{j}\mid Z_{j};\theta_{0},\delta_0\right)\\
			\nabla_{\delta}\log p\left(Y_{j}\mid Z_{j};\theta_{0},\delta_0\right)
		\end{array}\right)
		\right]
		-\left(\begin{array}{c}
			U_{\theta}(i/m)\\
			U_{\delta}(i/m)
		\end{array}\right)\right\Vert =o_{p}\left(1\right).\label{eq:gauss_proc}
	\end{equation*}
	where $U_{\theta}$ and $U_\delta$ are $p$- and $d$-dimensional, respectively.
	\label{assume:gauss}
\end{assumption}
\begin{assumption}
	Under the null, $\hat{\theta}_{m,i}$ is asymptotically linear with a remainder
	term that converges uniformly to zero, i.e.
	\begin{align*}
		\max_{m<i\le mK}\sqrt{m}\left\Vert \left(\hat{\theta}_{m,i}-\theta_{0}\right)-
		\mathbb{E}\left[-\sum_{j=1}^{i}
		\nabla_{\theta}^{2}\log p\left(Y_{j}\mid Z_{j};\theta_{0},\delta_0\right)
		\right]^{-1}
		\sum_{j=1}^{i}
		\nabla_{\theta}\log p\left(Y_{j}\mid Z_{j};\theta_{0},\delta_0\right)
		\right\Vert  & =o_{p}\left(1\right).
		\label{eq:theta_remain}
	\end{align*}
	\label{assume:remain}
\end{assumption}
\begin{assumption}
	Under the null, there exist functions $\Lambda_0: [1, K] \mapsto \mathbb{R}^{p\times p}$ and $\bar{V}_0: [1, K] \mapsto \mathbb{R}^{d\times p}$ such that
	\begin{align*}
		\max_{m<i\le mK}
		\left \|\Lambda_{0}^{-1}\left(\frac{i}{m}\right) -  m \mathbb{E}\left[-\sum_{j=1}^{i}
		\nabla_{\theta}^{2}\log p\left(Y_{j}\mid Z_{j};\theta_{0},\delta_0\right)
		\right]^{-1}
		\right \| = o_p(1)
		\\
		\bar{V}_0(t) = \mathbb{E}\left[
		\nabla_{\theta}\nabla_{\delta}\log p\left(Y_{\lfloor mt \rfloor }\mid Z_{\lfloor mt \rfloor };\theta_0,\delta_0\right)
		\right]
		\quad \forall t \in [1,K].
	\end{align*}
	\label{assume:matrices}
\end{assumption}
\noindent These assumptions hold, for instance, under piecewise local stationarity \citep{Wu2018-jk, Horvath2021-db}.
We can then prove that $\psi_m^{(\plugin)}$ is well-approximated by the process
\begin{align}
	\hspace{-0.6in}
	\phi_{m}(t_{1},t_{2})
	=\sum_{i=t_{1}}^{t_{2}}
	\nabla_{\delta}\log p\left(Y_{i}\mid Z_{i};\theta_{0},\delta_0\right)
	+\sum_{i=t_{1}}^{t_{2}}
	V_{0}\left(\frac{i}{m}\right ) \Lambda_{0}^{-1}\left(\frac{i-1}{m} \right)
	\sum_{j=1}^{i-1}
	\nabla_{\theta}\log p\left(Y_{j}\mid Z_{j};\theta_{0},\delta_0\right)
\end{align}
under the null and that the latter converges weakly, as formalized below.
\begin{theorem}
	Suppose the null hypothesis is true and that Assumptions~\ref{assume:gauss}, \ref{assume:remain}, and \ref{assume:matrices} hold.
	In addition, suppose the second and third derivatives are bounded (see Appendix for the formal assumption).
	Then
	\begin{align*}
		\max_{m<t_{1},t_{2}\le mK}
		\frac{1}{\sqrt{m}}
		\left\Vert \psi_{m}^{(\plugin)}(t_{1},t_{2})-\phi_{m}(t_{1},t_{2})\right\Vert =o_{p}(1).
	\end{align*}
	and
	\begin{align*}
		\hspace{-0.9in}
		\left\{ (\nu_{1},\nu_{2})\mapsto \frac{1}{\sqrt{m}} \psi_{m}^{(\plugin)}(\lfloor m\nu_{1} \rfloor, \lfloor m \nu_{2} \rfloor )\right\}_{(\nu_{1},\nu_{2}) \in \Delta}
		\Rightarrow
		\left\{ (\nu_{1},\nu_{2})\mapsto U_{\delta}(\nu_{2})-U_{\delta}(\nu_{1})+\int_{\nu_{1}}^{\nu_{2}}\bar{V}_{0}(v)\Lambda_{0}^{-1}(v)U_{\theta}(v)dv\right\}_{(\nu_{1},\nu_{2}) \in \Delta}
	\end{align*}
	where $\Delta = \left\{(\nu_1, \nu_2): \nu_1 < \nu_2, \nu_1  \in [1,K], \nu_2 \in [1,K] \right\}$.
	\label{thrm:score_conv}
\end{theorem}
\noindent In addition, Theorem~\ref{thrm:alt_consist} in the Appendix proves that the procedure is consistent if analogous assumptions hold under the alternative.

Given Theorem~\ref{thrm:score_conv}, we can determine the DCLs by analyzing the distribution of $\phi_m$ among the resampled sequences $\{ (Z_t, Y^{*(b)}_t) : t = m+1,\cdots, mK \}$.
One caveat is that this technically requires sampling from the true $\theta$, but we only have an estimate for its value.
As such, we perform the parametric bootstrap and resample the $t$-th outcome using the MLE estimated up to time $t-1$.
The full monitoring procedure is outlined in Algorithm~\ref{algo:dcl}.
Additional implementation details are provided in the Appendix.

\begin{algorithm}
	\caption{Pseudocode for score-based CUSUM procedure with dynamic control limits}
	\label{algo:dcl}
	\begin{algorithmic}
		\State Select time factor $K$, alpha spending function $\alpha^{\rel}$, and number of bootstrap sequences $B$.
		\State Let $\mathcal{B}_m = \{1,...,B\}$ represent the bootstrap sequences that have not been rejected at time $m$.
		\State Observe non-contaminated data $\{(Z_t, Y_t):t = 1,\cdots, m\}$.
		\State Calculate MLE $\hat{\theta}_{m,m}$
		\For{$b = 1,...,B$}
		\State Resample outcome $Y_t^{*(b)}$ given $Z_t$ for $t = 1,\cdots, m$, with $\theta=\hat{\theta}_{m,m}$ and $\delta = 0$.
		\EndFor
		\For{$t = m+1,...,mK$}
		\State Observe $(Z_t, Y_t )$.
		\State Calculate chart statistic $C_{m}^{(plugin)}(t)$ and MLE $\hat{\theta}_{m,t}$.
		\For{$b \in \mathcal{B}_{t - 1}$}
		\State Resample outcome $Y_{t}^{*(b)}$ given $Z_t$ with $\theta = \hat{\theta}_{m,t - 1}$ and $\delta = 0$.
		\State Compute $\phi_m(t', t)$ for $t' = m+1, \cdots, t - 1$ for the $b$-th bootstrap sequence.
		\State Calculate $C^{*,(b)}_{m}(t) = \max_{t' \in \{m+1, \cdots, t\}} \phi_m(t', t)$.
		\EndFor
		\State Set $h_{m}(t)$ such that the proportion of bootstrap chart statistics exceeding the DCL is \[\left | \left\{b : b \in \mathcal{B}_{t - 1}, C^{*,(b)}_{m}(t) > h_{m}(t) \right\} \right |/B = \alpha^{\rel}\left(t/m \right) - \alpha^{\rel} \left((t-1)/m \right).\]
		\State Define $\mathcal{B}_{t}= \{b : b \in \mathcal{B}_{t - 1}, C^{*,(b)}_{m}(t) \le h_{m}(t)\}$.
		\If{$C_{m}^{(\plugin)}(t) > h_{m}(t)$}
		\State Fire an alarm. Break.
		\EndIf
		\EndFor
	\end{algorithmic}
\end{algorithm}

\section{Simulation studies}
We now investigate how the score-based CUSUM is influenced by various factors, including continual retraining of the ML algorithm, magnitude of the structural change, and clinician trust.
We vary each factor individually to isolate its impact, though many factors can co-occur in real-world settings.

Details for all empirical experiments are in the Appendix.
In all simulations, the outcome $Y_t(0)$ given predictors $X_t \in \mathbb{R}^{p'}$ and $\tilde{X}_t \in \mathbb{R}$ is generated from a logistic regression model.
Unless specified otherwise, we set $p'=8$ and train the risk prediction algorithm using logistic regression.
The nominal false alarm rate for the score-based CUSUM is set to $\alpha= 0.1$.

For comparison, we also implement Bayesian monitoring using \texttt{Stan} \citep{Carpenter2017-ze}, which performs Hamiltonian Monte Carlo (HMC).
Additional implementation details are in the Appendix.
While our findings demonstrate that Bayesian monitoring can yield similar results for appropriately chosen priors, a number of major limitations must be addressed before it can be recommended for practical use.
First, there are currently no ready-to-use software packages for long-term monitoring of the model \eqref{eq:mdl_de}.
HMC is not designed for sequential monitoring because it runs posterior inference from scratch for every observation.
In our experiments, HMC ran on the order of hours whereas the score-based CUSUM ran on the order of seconds.
Second, posterior inference is generally difficult for models with structural constraints such as \eqref{eq:mdl_tc}, so we resorted to approximate inference when using this model.
Finally, Bayesian inference can be sensitive to misspecification of the model and/or prior, so more concrete recommendations are needed on their selection.

\subsection{False alarm rate control}
\label{sec:sim_false}
We begin with evaluating false alarm rate control of score-based CUSUM monitoring in finite samples.
The data are simulated under the null and satisfy either the assumption of conditional exchangeability or time-constant selection bias.
A shift in the treatment propensities is introduced halfway through the monitoring period.

Here we consider a locked ML algorithm.
In the \texttt{Conditional Exchangeability} simulation, the treatment propensities are generated according to a logistic regression model with only $\hat{f}(X_t)$ and $\tilde{X}_t$ as inputs; the outcome is generated according to \eqref{eq:mdl_de}.
In the \texttt{Time-constant selection bias} simulation, the treatment propensities and outcome are generated per Example~\ref{example:equiconf} of the Appendix.
We consider two versions of both assumptions: one that only conditions on $\hat{f}(X_t)$ and another that conditions on $(\hat{f}(X_t), \tilde{X}_t)$.
The former leads us to test \eqref{eq:cmi_hypothesis} and the latter leads us to test \eqref{eq:cmi_hypothesis_extend}.

\begin{figure}
	\centering
	\textit{Conditional exchangeability}\\
	\includegraphics[width=0.3\linewidth]{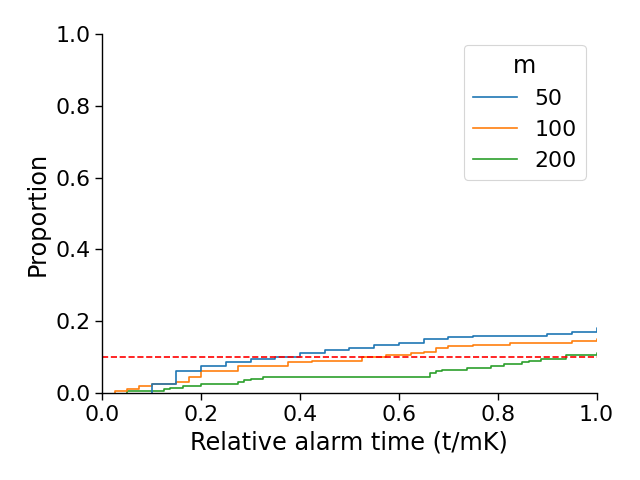}
	\includegraphics[width=0.3\linewidth]{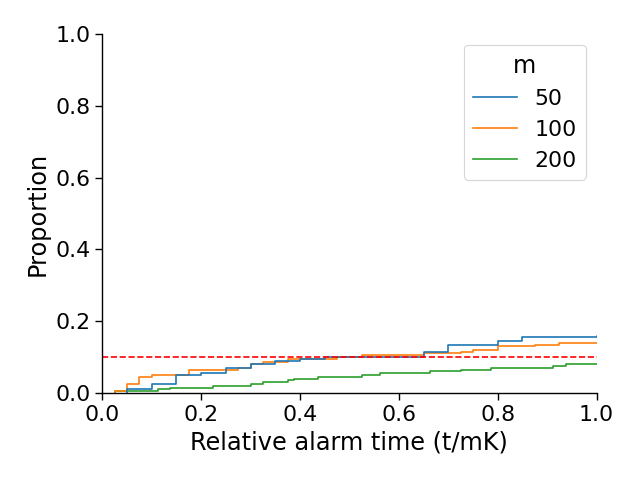}
	
	\textit{Time-constant selection bias}\\
	\includegraphics[width=0.3\linewidth]{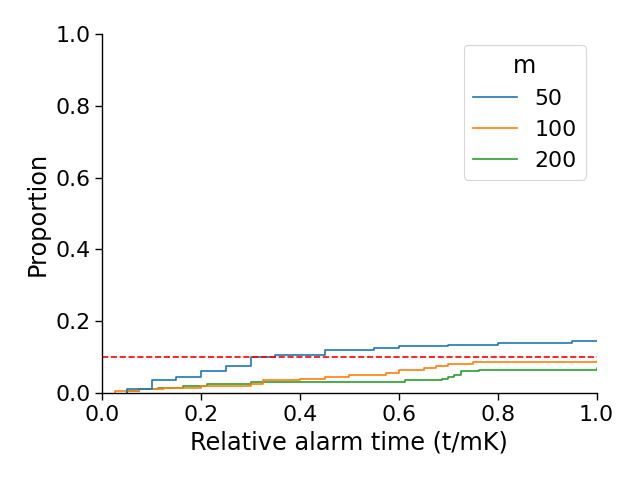}
	\includegraphics[width=0.3\linewidth]{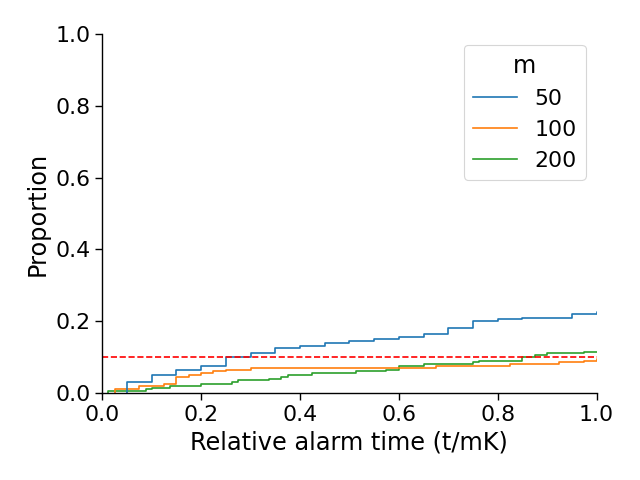}
	\caption{
		Cumulative distribution of alarm times for score-based CUSUM monitoring of a locked model $\hat{f}$ under the null.
		The assumption of conditional exchangeability and time-constant selection bias are satisfied in the top and bottom rows, respectively.
		In the left column, the assumptions hold when conditioning on $\hat{f}(X_t)$; in the right column, the assumptions hold when conditioning on $(\hat{f}(X_t), \tilde{X}_t)$.
		The target false alarm rate is 0.1, which is achieved as the size of the non-contaminated dataset $m$ increases.
	}
	\label{fig:type_i}
\end{figure}

We vary the size of the noncontamination dataset size $m$ and monitor for $mK$ time points, with $K$ set to $4$.
As shown in Figure~\ref{fig:type_i}, Type I error is inflated for small values of $m$, but converges to the nominal rate once $m$ is sufficiently large.

\subsection{Monitoring a continually retrained model in stationary settings}
\label{sec:sim_retrain}
An important safety requirement for continually retrained models is that they remain well-calibrated over time.
By wrapping such models within a monitoring procedure, we can continually check that this requirement is met while allowing model discrimination to change time.
This approach is particularly beneficial for online learning procedures that only provide weak (or no) performance guarantees, such as black-box models.
Here we investigate the alarm rates for continually retrained ridge-penalized logistic regression models and gradient boosted trees (GBT) in the stationary setting where data is independently and identically distributed.
This simulation also simultaneously investigates the sensitivity of our monitoring procedure to misspecification of the conditional distribution, since the monitoring model is unlikely to hold uniformly over time.

We simulate a higher-dimensional setting with $p'=50$ and a propensity model that satisfies conditional exchageability.
Both risk prediction algorithms were continually retrained on all prior SOC data for a fixed set of hyperparameters.
Because GBT can be poorly calibrated, we recalibrated model updates using Platt scaling.
The monitoring procedures were implemented to detect shifts on the logit and risk scales using \eqref{eq:mdl_de} and \eqref{eq:mdl_tc}, respectively.

\begin{figure}
	\centering
	\textit{Ridge-penalized logistic regression}\\
	\includegraphics[width=0.35\textwidth]{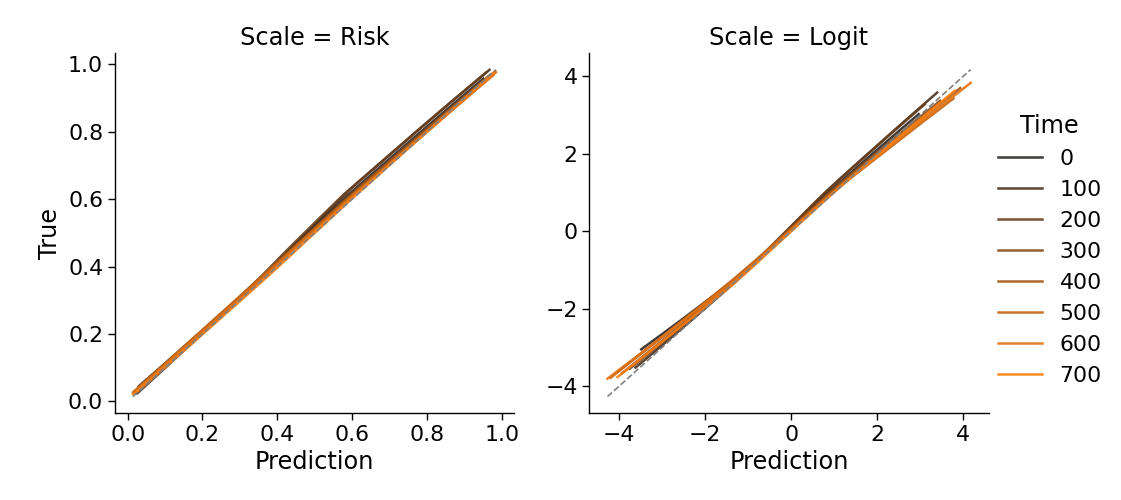}
	\includegraphics[width=0.2\textwidth]{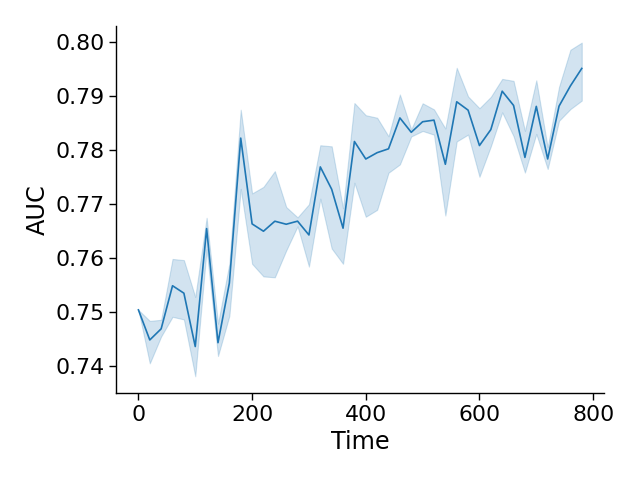}
	\includegraphics[width=0.2\textwidth]{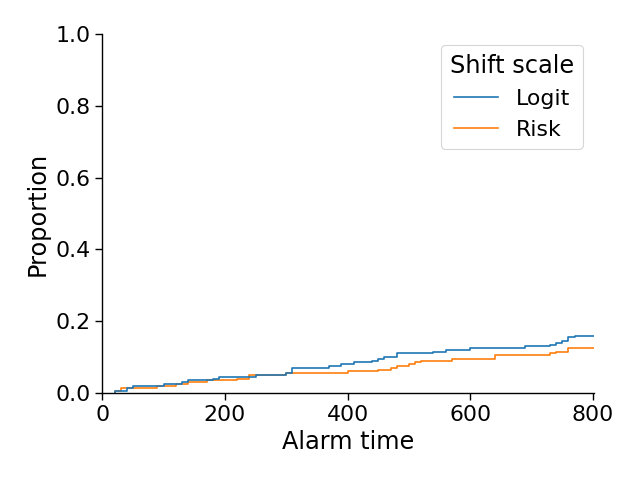}
	\includegraphics[width=0.2\textwidth]{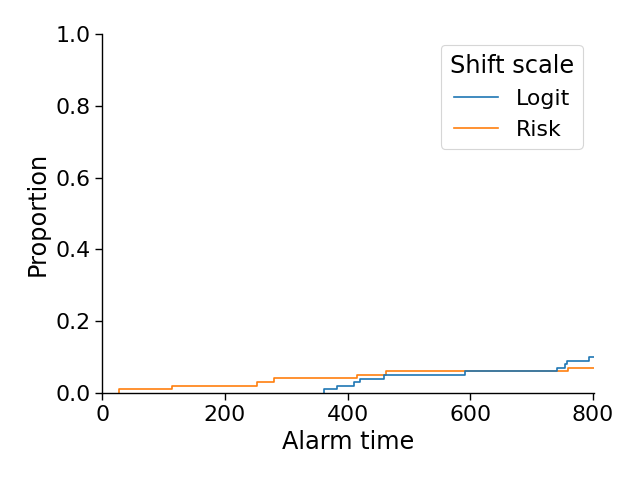}
	
	\textit{Gradient boosted trees with Platt scaling}\\
	\includegraphics[width=0.35\textwidth]{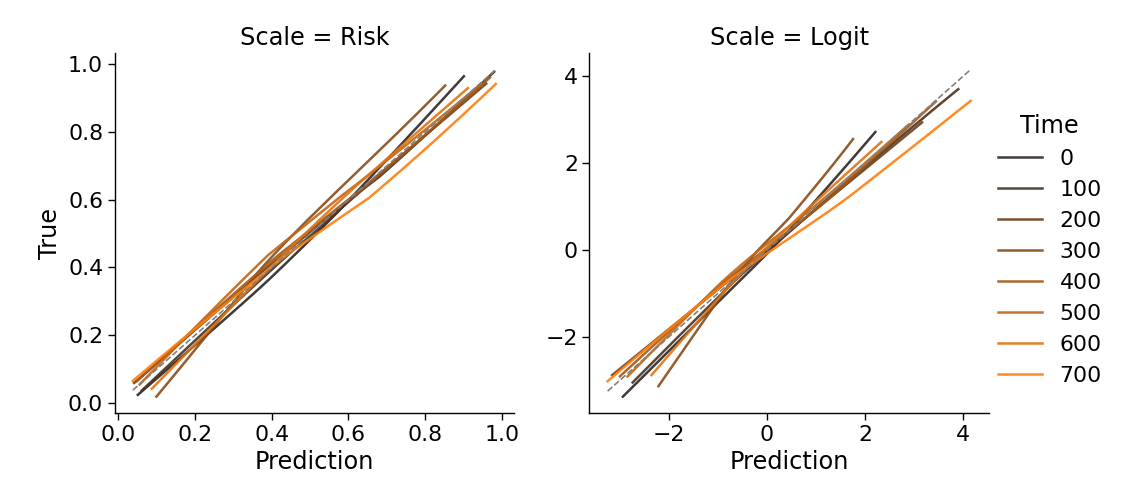}
	\includegraphics[width=0.2\textwidth]{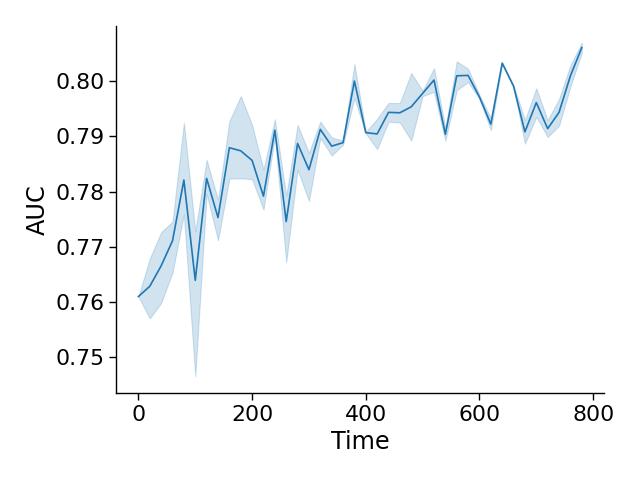}
	\includegraphics[width=0.2\textwidth]{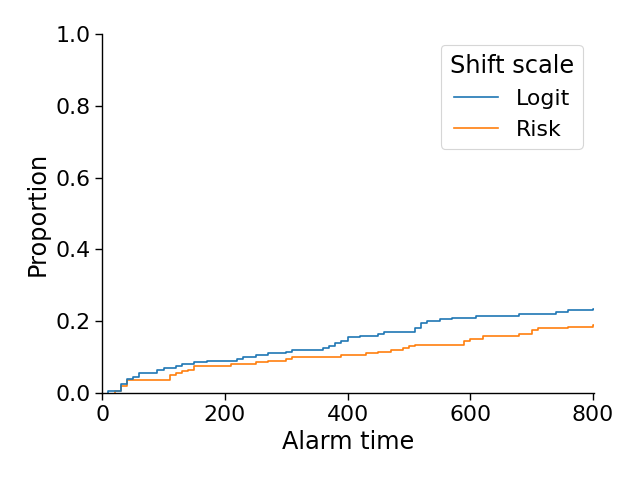}
	\includegraphics[width=0.2\textwidth]{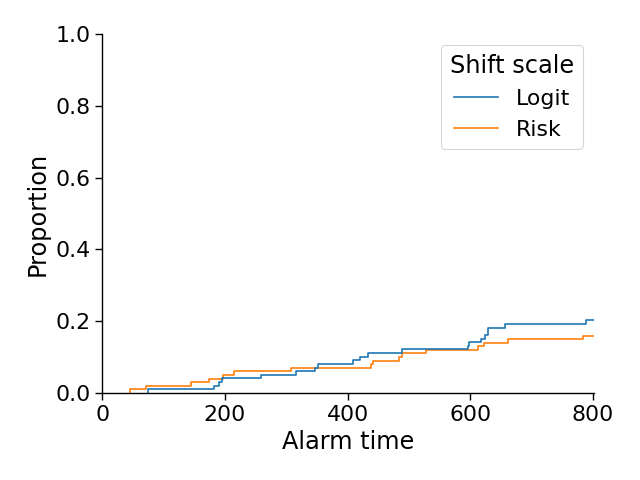}
	
	\caption{
		Monitoring continually retrained ML risk prediction algorithms using ridge-penalized logistic regression (top) and gradient boosted trees with Platt scaling (bottom).
		Calibration curves of model updates from a single replicate are plotted on the risk scale and logit scale in the first and second columns, respectively.
		The average AUC of the model updates are shown in the third column.
		The cumulative distribution of the alarm times are shown for score-based CUSUM monitoring (fourth column) and Bayesian changepoint monitoring (fifth column), which were implemented to detect for shifts on either the risk scale or the logit scale.
	}
	\label{fig:monitor_retrained_null}
\end{figure}

For ridge-penalized logistic regression, the CUSUM alarm rate is close to the nominal rate of $\alpha = 0.1$ when monitoring on the risk scale and increases to 0.16 when monitoring on the logit scale (Figure~\ref{fig:monitor_retrained_null}).
Similar trends were observed when using Bayesian monitoring.
The difference in alarm rates is explained by the calibration curves of the model updates: these curves coincide nearly perfectly with the ideal diagonal line on the risk scale but deviate slightly on the logit scale.
As such, we recommend monitoring shifts on the risk scale since (i) the risk scale is more meaningful in practice and (ii) it is less sensitive to model misspecification.
These results are highly promising, since it shows how the false alarm rate for a continually updated model can actually match that of a locked model and allow for steady improvement in model discrimination over time.

For GBT, we found that Platt scaling improved model calibration but was unable to maintain perfect calibration uniformly over time.
Using the CUSUM, alarm rates were 0.19 and 0.24 when monitoring on the risk and logit scales, respectively; results from Bayesian inference were similar.
This highlights how one-time model recalibration is insufficient for maintaining model reliability.
Instead, one may need online calibration methods such as \citep{Feng2022-wf}.

\subsection{Big shifts. Little shifts. Locked models. Evolving models.}
\label{sec:magnitude}
We now investigate alarm rates in the presence of a structural change.
We vary both the magnitude of the structural change and examine how model retraining may affect alarm rates.
In simulations \texttt{large-shift} and \texttt{little-shift}, we shrink the coefficients of the outcome model at $t=50$ by 80\% and 50\%, respectively.
To retrain the ML model in nonstationary settings, we use an exponentially weighted average forecaster (EWAF), a classic online learning algorithm that can adapt to adversarial dataset shifts \citep{Cesa-Bianchi2006-tl, Feng2021-mf}.

\begin{figure}
	\centering
	\textit{Big Shift}\\
	\includegraphics[width=0.22\linewidth]{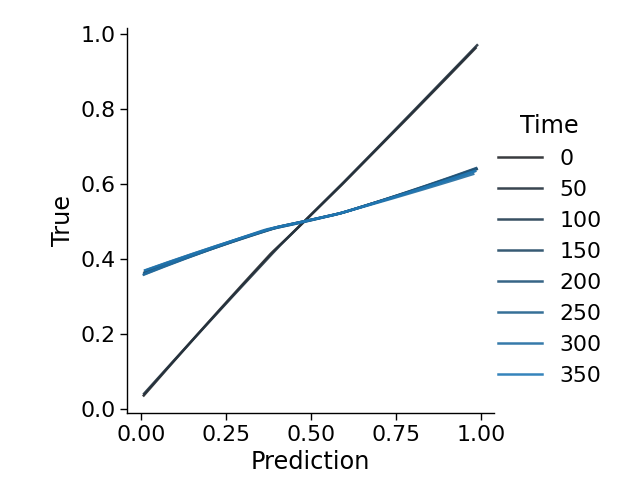}
	\includegraphics[width=0.22\linewidth]{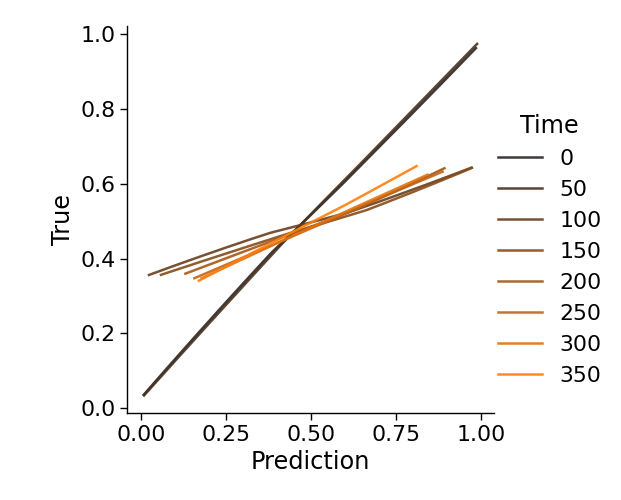}
	\includegraphics[width=0.24\linewidth]{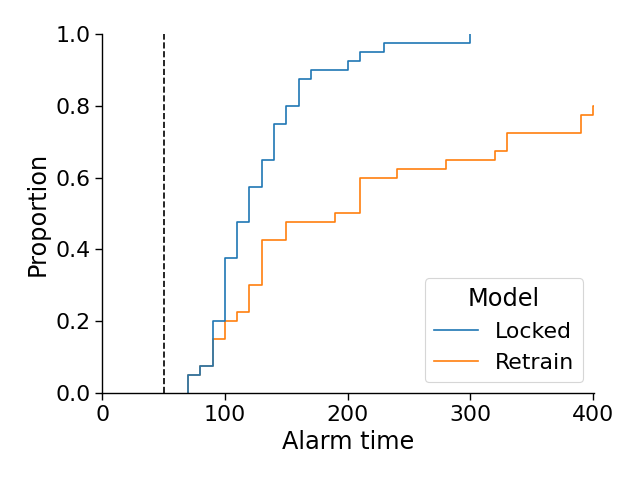}
	\includegraphics[width=0.24\linewidth]{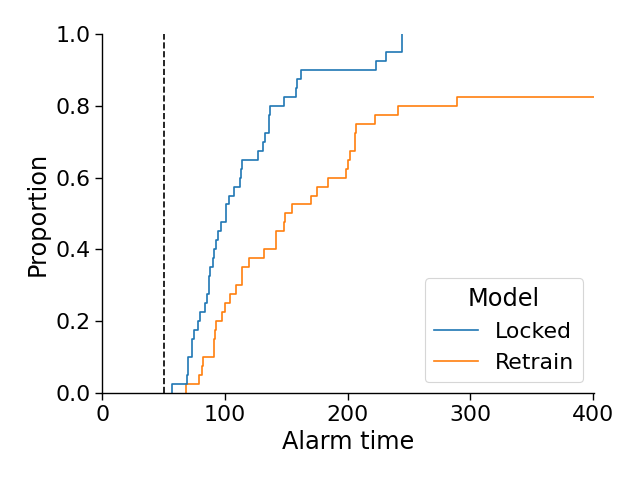}
	
	\textit{Small Shift}\\
	\includegraphics[width=0.22\linewidth]{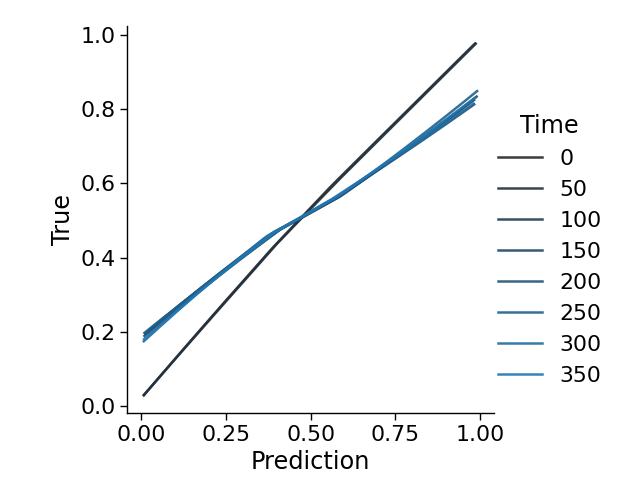}
	\includegraphics[width=0.22\linewidth]{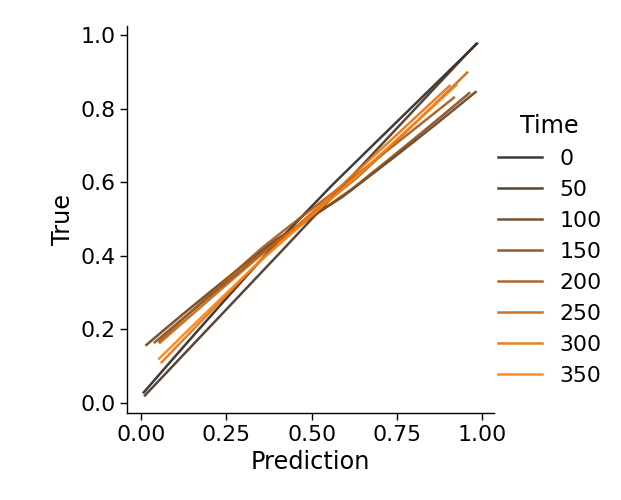}
	\includegraphics[width=0.24\linewidth]{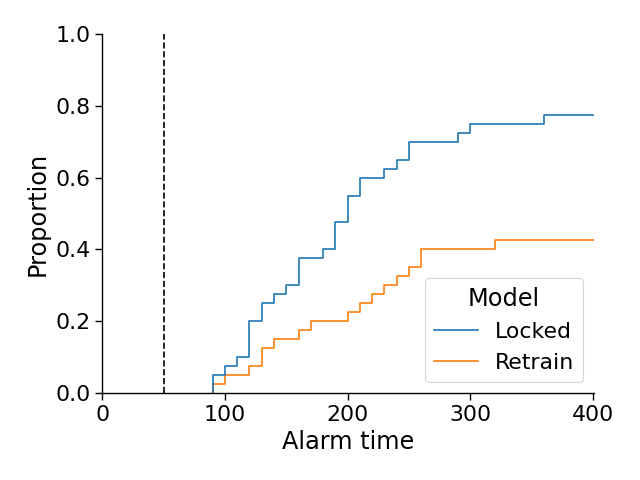}
	\includegraphics[width=0.24\linewidth]{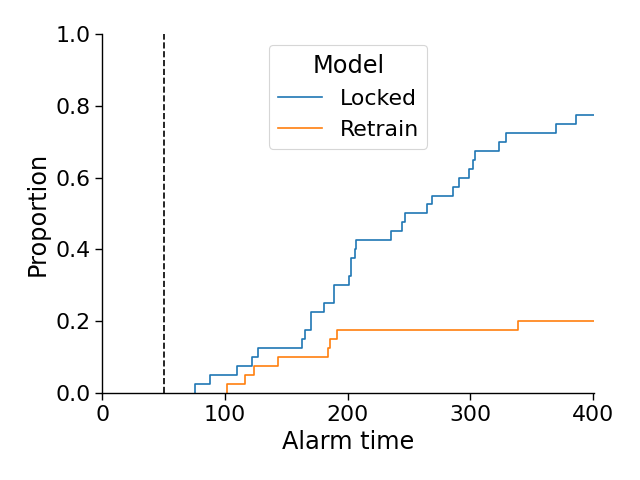}
	
	\caption{
		Monitoring locked versus continually retrained models in the presence of big and small distribution shifts (top and bottom rows, respectively).
		The calibration curves of the locked and continually retrained models (first and second columns, respectively) are plotted over time, including before and after the changepoint.
		The third and fourth columns show the cumulative distribution of alarm times for score-based CUSUM monitoring and Bayesian changepoint monitoring, respectively.
		The changepoint is indicated in the right two columns by a dashed vertical line.
	}
	\label{fig:magnitude}
\end{figure}

Across all scenarios, we find that the time to an alarm is longer if the magnitude of the structural change is smaller (Figure~\ref{fig:magnitude}).
Perhaps more interestingly, we consistently find that model retraining increases the time to an alarm, depending on the quality of the model updates.
In the \texttt{large shift} simulation, the median time to alarm is delayed by 80 observations, because the number of observations needed to detect a shift is small whereas that needed to restore model calibration is much larger.
In contrast, the alarm rate drops to below 50\% with model retraining in the \texttt{small-shift} simulation, because detection requires many more observations and the EWAF was often able to recalibrate the model during this time.
Similar patterns can be seen when using Bayesian monitoring.
In summary, the alarm time can be viewed as the result of a competition between model monitoring and updating.
By designing a sufficiently fast and adaptive model updating procedure, one can significantly extend the total lifetime of a risk prediction algorithm.

\subsection{Clinician trust}
\label{sec:trust}

Intuitively, clinician trust can interfere with our ability to detect performance decay.
To investigate this, we simulate three levels of clinician trust.
In the \texttt{No-trust} simulation, the treatment propensity is uncorrelated with the model's prediction.
In \texttt{Calibrated-trust}, the rate of treatment is the same as the predicted risk.
In \texttt{Over-trust}, the decision to treat essentially thresholds on the model's predictions, so that nearly everyone above 50\% predicted risk is treated and nearly everyone below 50\% is untreated.
We also simulate two types of structural change: one in which shifts in the risk are symmetric (\texttt{symmetric-shift}) and one in which risks shift the most among patients with the highest initial risk (\texttt{high-risk}).

\begin{figure}
	\centering
	\textit{Shifts are largest in the high- and low-risk populations}\\
	\includegraphics[width=0.28\linewidth]{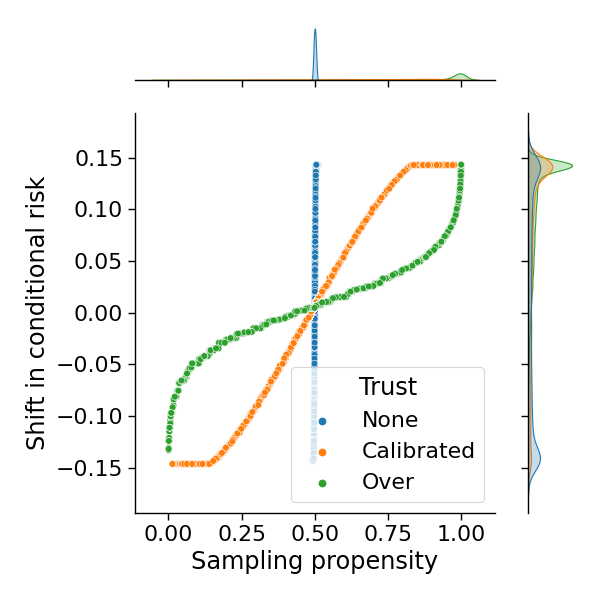}
	\includegraphics[width=0.32\linewidth]{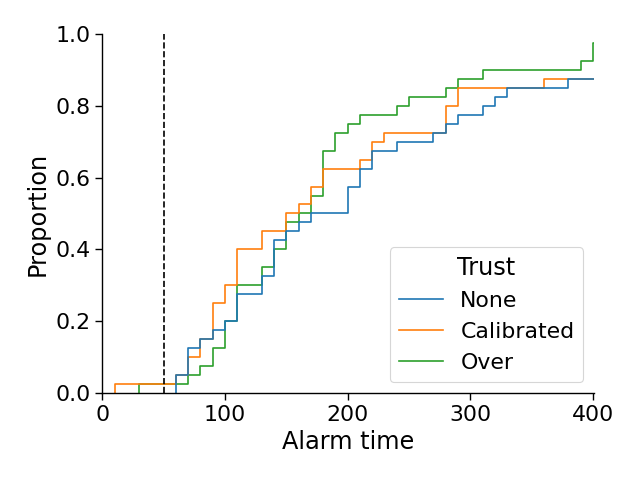}
	\includegraphics[width=0.32\linewidth]{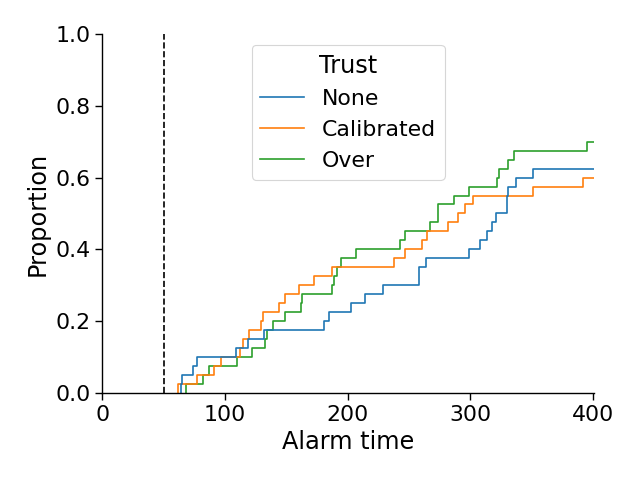}
	
	\textit{Shifts are largest in only the high-risk population}\\
	\includegraphics[width=0.28\linewidth]{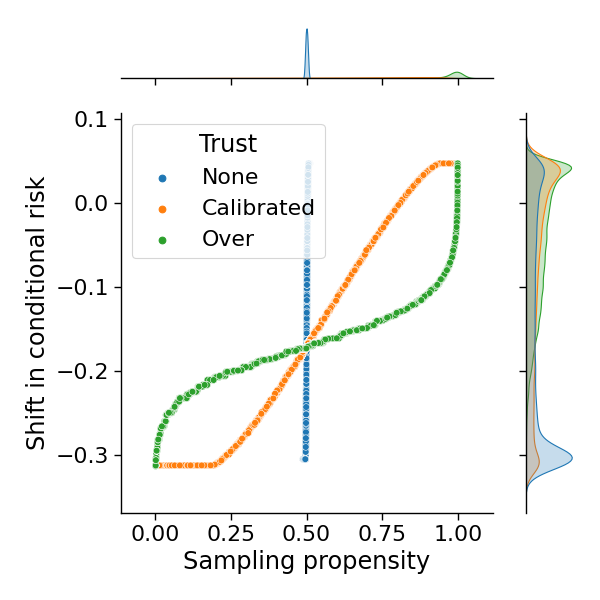}
	\includegraphics[width=0.32\linewidth]{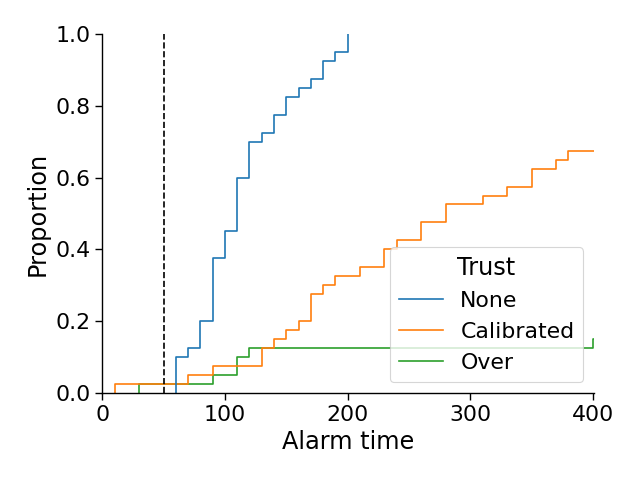}
	\includegraphics[width=0.32\linewidth]{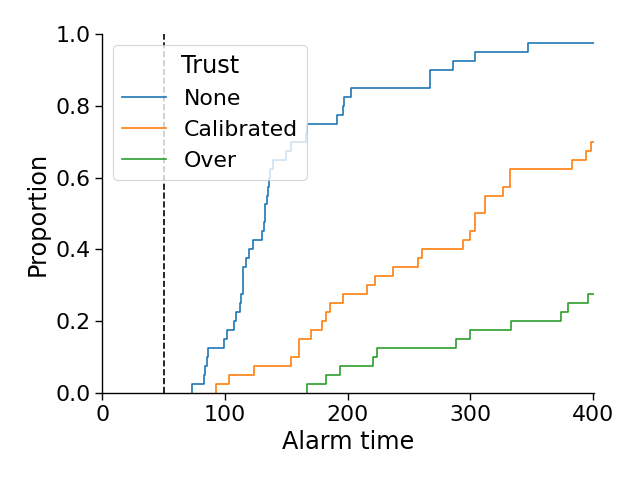}
	\caption{
		We simulate no trust, calibrated trust, and over trust in the ML algorithm to understand its impact on detection delay.
		We simulate two types of structural changes in the outcome model: shifts in the conditional risk that is symmetric and largest among patients with the lowest and highest risks (top row); and shifts that are largest only among patients with the highest risks (bottom row).
		The left column shows how shifts in the conditional risk vary with respect to the probability of a patient being assigned SOC, and thus their probability of being sampled for monitoring.
		The middle and right columns show the cumulative distribution of alarm times using score-based CUSUM and Bayesian monitoring, respectively.
	}
	\label{fig:trust}
\end{figure}

In the \texttt{symmetric-shift} simulation, increasing clinician trust had little impact on alarm times, and may have even decreased detection delay (Figure~\ref{fig:trust} top).
On the other hand, increasing clinician trust substantially increased detection delay in the \texttt{high-risk} simulation (Figure~\ref{fig:trust} bottom).
This difference is explained by the distribution of risk shifts in the monitored patient population.
In the \texttt{symmetric-shift} simulation, increasing trust tended to increase the representation of patients with lower predicted risks, which corresponds to those experiencing bigger shifts in their risk.
Whereas in the \texttt{high-risk} simulation, increasing trust increased the representation of subjects experiencing smaller shifts.

These results have an important practical implication: when designing monitoring strategies, we should incorporate any prior knowledge regarding which populations are likely to experience distribution shifts.
If changes are likely to be concentrated among subjects with high treatment propensities, passive monitoring of SOC-only data may not sufficient.
Instead, one may need to actively augment the amount of information in the monitoring data.
Prior work has suggested educating healthcare providers on the appropriate use of ML and warning against over-reliance \citep{Finlayson2021-ad} or randomizing patients to receive no ML-based recommendations when it is deemed ethical \citep{Harris2022-vi}.

\subsection{Sensitivity analysis of the time-constant selection bias assumption}
\label{sec:violate}

Finally, we explore how violations of the time-constant selection bias assumption can impact detection delay of a structural change.
We first simulate data that satisfies this time-constancy assumption based on Example~\ref{example:equiconf}. 
Then we introduce violations of this assumption by adding an edge from the unmeasured confounder $U_t$ to the final treatment decision $A_{t}$ in the DAG and setting a non-zero edge weight at times $t = 100$ or $300$.
Such shifts in the propensity model could occur if, say, a clinician suspects performance of the ML algorithm has decayed and decides to place more weight in the unmeasured risk factor $U_t$. 

Because the simulated violation is designed to dampen the shift observed in the data, we find that detection delay increases as the violation occurs earlier in time (Figure~\ref{fig:timeconstant_power}).
In the worst case scenario, the structural change \textit{and} the shift in the treatment propensities occurs at the same time ($t = 100$) and power drops by 30\%.
Nevertheless, such a scenario is unlikely to happen in practice since it assumes clinicians know exactly when performance decays.

In this simulation, we find that the power of Bayesian monitoring is much lower than that for the score-based CUSUM.
This is likely due to the sensitivity of Bayesian inference to model misspecification: the monitoring model assumes a single changepoint, whereas the observed data distribution shifts at two time points.
Consequently, its power drops by over 40\%.
In fact, even without violations of the time-constancy assumption, the power of the Bayesian procedure is much lower than that of the score-based CUSUM.
This may be due to difficulties in performing posterior inference for \eqref{eq:mdl_tc}, which is only partially differentiable.

\begin{figure}
	\centering
	\includegraphics[width=0.33\textwidth]{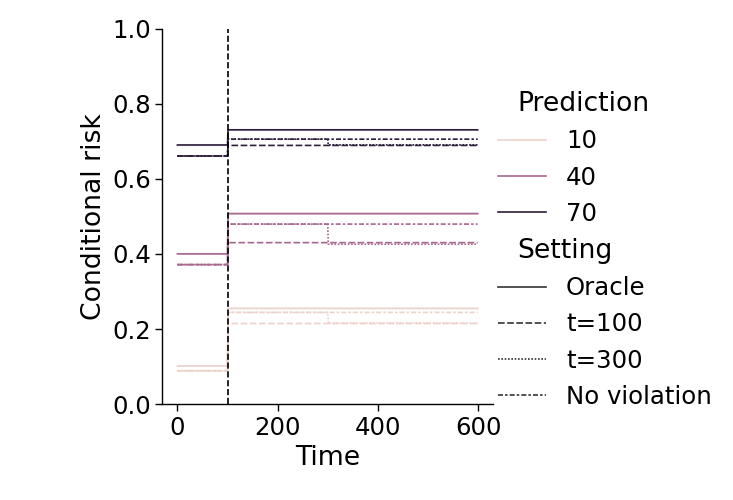}
	\includegraphics[width=0.3\textwidth]{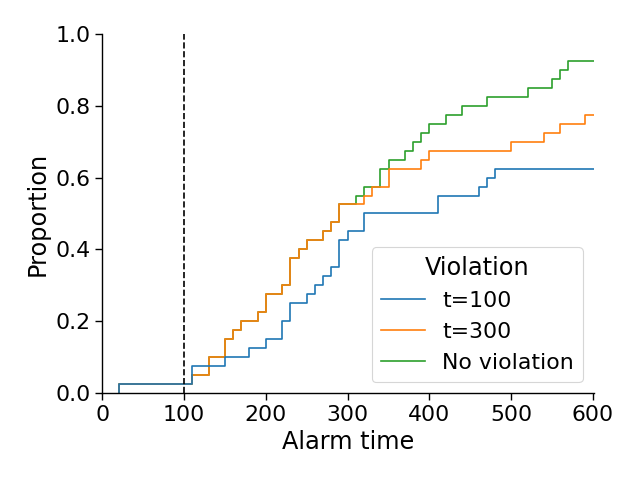}
	\includegraphics[width=0.3\textwidth]{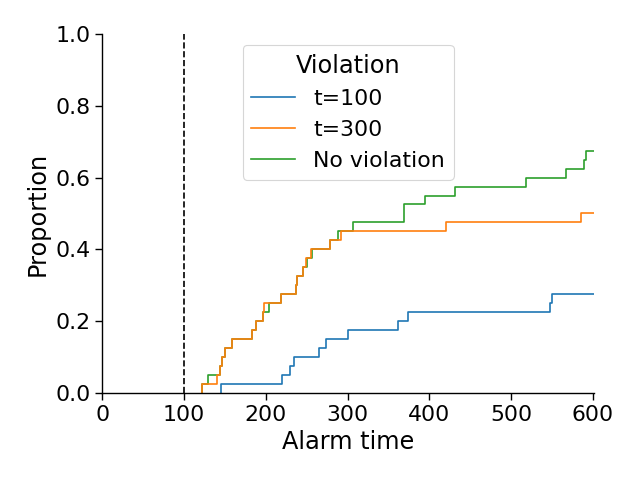}
	
	\caption{
		Monitoring a structural change at time $t = 100$ (dashed vertical line) when the assumption of time-constant selection bias is violated.
		We introduce a shift in the treatment propensities that violates the time-constant selection bias assumption at times $t = 100$ and $t=300$.
		We also simulate no violation of the time-constant selection bias by never introducing this shift in the treatment propensities.
		For different risk prediction values, we plot the conditional risks among the SOC-only population over time, which are biased for those among the general population.
		We also plot the oracle conditional risk for comparison.
		The middle and right columns show the cumulative distribution of alarm times for score-based CUSUM monitoring and Bayesian monitoring, respectively.
	}
	\label{fig:timeconstant_power}
\end{figure}

\section{Monitoring a ML-based PONV risk calculator}

Postoperative nausea and/or vomiting (PONV) is one of the most common side effects of anesthesia.
To reduce PONV rates, various risk prediction models have been developed to guide the use of antiemetics. 
Here we simulate monitoring a ML-based PONV risk calculator based on retrospective data from the UCSF MPOG registry ($n = 2434$).
We define a patient as receiving additional care if they received at least two antiemetics. 

Using data from January 2018 to May 2019, we trained a random forest (RF) to predict risk of PONV based on preoperative variables, including biological sex, smoking status, age, ASA score, and blood test results.
We then locked the model, used the first 200 patients to initialize the monitoring procedures, and started monitoring mid-December 2019. 
We suppose the conditional exchangeability assumption holds with respect to $\hat{f}_t(X_t)$ and modeled the data using \eqref{eq:mdl_de}.
Control limits were set so that the false alarm rate/probability is 20\%.

\begin{figure}
	\centering
	\textit{Monitoring a locked model}\\
	\includegraphics[width=0.35\textwidth]{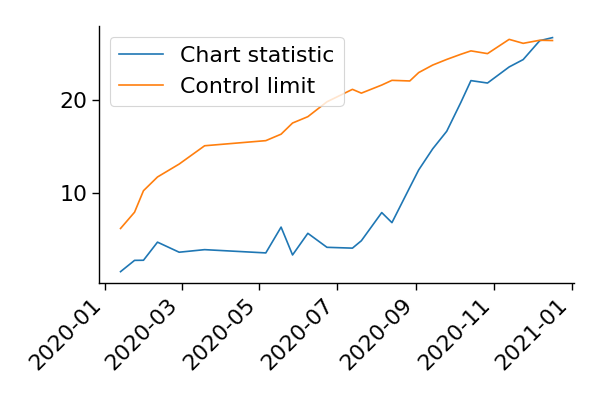}
	\includegraphics[width=0.35\textwidth]{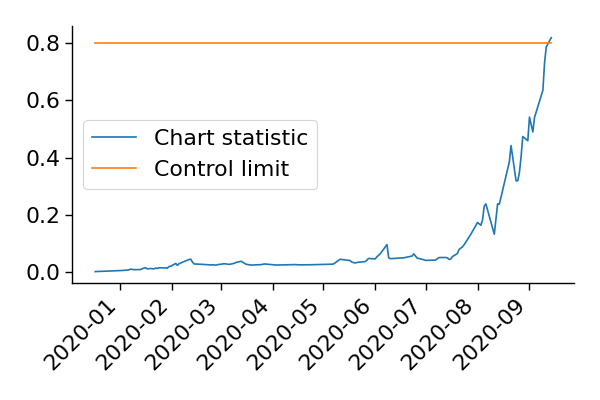}
	
	\textit{Monitoring a continually updated model}\\
	\includegraphics[width=0.35\textwidth]{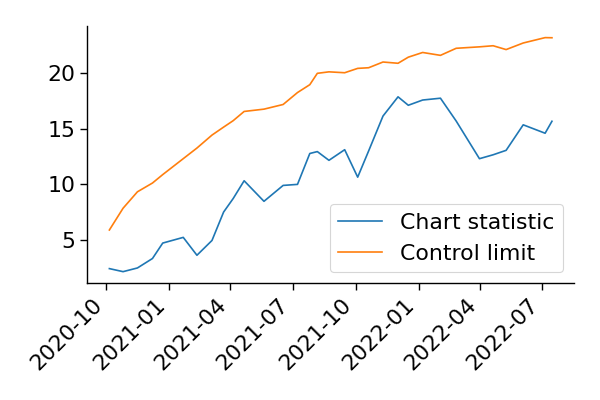}
	\includegraphics[width=0.35\textwidth]{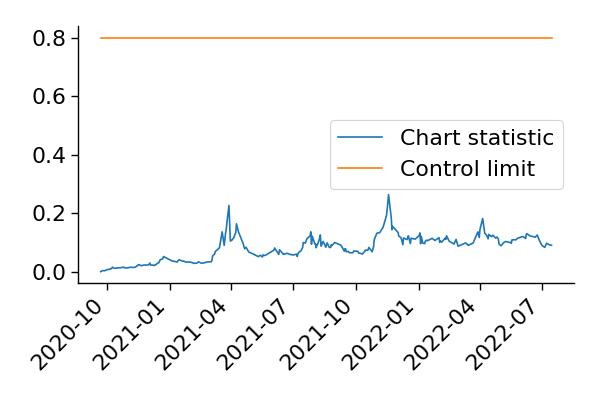}
	
	\includegraphics[width=0.35\textwidth]{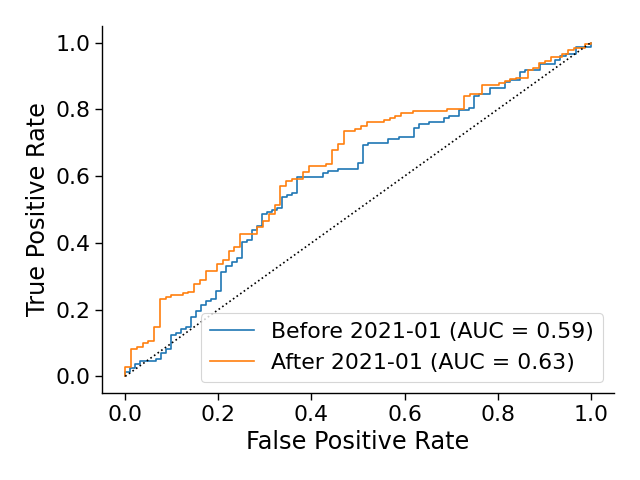}
	\includegraphics[width=0.35\textwidth]{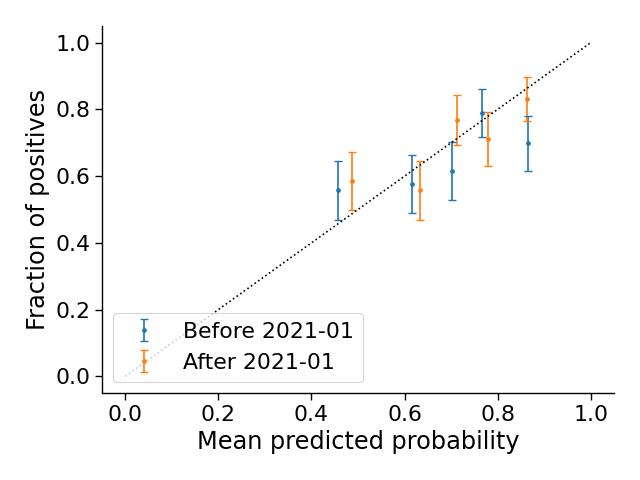}
	
	\caption{
		Control charts for monitoring a ML-based risk prediction model for Post-operative Nausea and Vomiting (PONV) using score-based monitoring (left) and Bayesian inference (right).
		The top control charts monitor a locked model, starting before the COVID pandemic.
		The control charts in the second row monitor the continually retrained model, starting midway through the pandemic.
		The ROC and calibration curves of the averaged models before and after January 2021 are shown in the bottom row.
	}
	\label{fig:ponv}
\end{figure}

The score-based and Bayesian monitoring procedures fired alarms in late 2020 (Figure~\ref{fig:ponv} top).
Both control charts suggest the shift started in May 2020, as their chart statistics started to increase at this time.
The detection of a performance shift during the COVID-19 pandemic is not unexpected and may be explained by many causes: (i) the anesthesia department had implemented changes in antiemetic medication administration; (ii) there was a big shift in the type of patients who received surgery during this time because of the pandemic; and (iii) exposure to the SARS-Cov2 virus affected the overall health of many patients.
So if this ML algorithm was indeed deployed in the hospital, it would be important to have a dedicated quality assurange/improvement (QA/QI) team to investigate what the root causes are and determine which corrective actions to take.

We then considered a scenario where the RF was retrained every 10 observations, where we split the data stream with 40\% dedicated for model retraining and 60\% for model monitoring.
We initialized the monitoring procedure using observations from 200 patients starting July 2019.
Due to data splitting, initialization spanned a longer time period and monitoring began October 2020.
Although the AUC of the RF dropped to 0.59 in the beginning of the pandemic due to shifts in the data, continual model retraining improved the AUC to 0.63 (Figure~\ref{fig:ponv} bottom).
Moreover, no shifts in the calibration curve were detected during this time.

\section{Related work}
The problem of CMI is closely related to the problem of verification bias \citep{Begg1983-bq, Zhou1994-ye, Alonzo2005-la}, which arises when disease status is only verified for a subset of patients.
Verification bias is typically presented using a missing data framework.
However, given the fundamental link missing data and causal inference,  the assumptions and techniques used to address CMI and verification bias share many similarities.
For instance, our assumption of conditional exchangeability can be viewed as generalization of the missing at random assumption.

Given this connection with the missing data literature, our work is also closely related to monitoring time-to-event data \citep{Sego2009-qd, Gandy2010-zr, Sun2014-yh}, many of which make similar conditional independence assumptions.
However, prior work assume that the pre-change distribution is known, that the structural change has a very specific form that is summarized by a single parameter (e.g. a proportional increase in the hazard), and that the sequence of predictors are stationary.
None of these necessarily hold in our setting.
Recently, \citet{Xue2021-js} proposed a batch-sequential procedure for monitoring time-to-event data that relaxes many of these assumptions; however, it relies on large batch sizes to justify asymptotic approximations.
In contrast, the score-based CUSUM procedure can be used in a fully sequential manner.

Finally, we note that this work is part of the growing literature on sequential testing of causal quantities using observational data \citep{Li2011-bq, Cook2015-pd, Waudby-Smith2021-rz}.
Unlike these works, our goal is test for structural change rather than the value of a time-constant parameter.

\section{Discussion}

Although CMI can complicate performance monitoring of risk prediction models, we have shown that CMI is ignorable when monitoring conditional performance measures under either the assumption of conditional exchangeability or time-constant selection bias.
We introduce a new score-based CUSUM procedure with DCLs that provides false alarm rate control, even if the ML algorithm and/or the clinician's interactions with the algorithm evolve over time.

Whereas prior work on performance monitoring has been restricted to locked algorithms \citep{Feng2022-mk}, this work extends these ideas to monitor continually updated algorithms.
By wrapping online learning methods within a monitoring framework, we greatly strengthen the performance guarantees of these algorithms.
In particular, many online learning methods only control average performance over long time periods and do not protect against sudden performance drops  \citep{Cesa-Bianchi2006-tl, Feng2021-mf}.
Using monitoring procedures, we can address the latter.

Many issues warrant further investigation.
Although we found the score-based CUSUM to be somewhat robust to model misspecification, it will be important to further improve robustness of the method.
We have also assumed that the treatment option is always clearly delineated and that patient outcomes are observed immediately.
This may not hold in certain settings, so future work should consider more complex treatment variables and account for delays in reporting.
Finally, monitoring using SOC-only data is most powerful when treatment decisions do not mask areas of major performance decay.
Future work should explore ways to address extreme treatment propensities, such as the inclusion of other data sources and randomization.

\section*{Acknowledgments}
This work was greatly improved by helpful suggestions and feedback from Alan Hubbard, Fan Xia, and Si Wen. We are grateful to Daniel Lazzareschi for sharing the UCSF MPOG data.

\section*{Funding}
This work was supported by the Food and Drug Administration (FDA) of the U.S. Department of Health
and Human Services (HHS) as part of a financial assistance award Center of Excellence in Regulatory
Science and Innovation grant to University of California, San Francisco (UCSF) and Stanford University,
(U01FD005978). The contents are those of the author(s) and do not necessarily represent the official views
of, nor an endorsement, by FDA/HHS, or the U.S. Government.

\bibliographystyle{plainnat}
\bibliography{main}

\pagebreak

\appendix

\begin{table}
	\caption{Mathematical symbols}
	\label{tab:symbols}
	\centering
	\begin{tabular}{c|p{3.5in}}
		Symbol & Meaning \\
		\toprule
		$X_t$ & Patient covariates that go into the ML algorithm\\
		$\tilde{X}_t$ & Additional patient covariates used to make treatment decisions\\
		$Y_t$ & Patient outcome\\
		$A_t$ & Treatment assignment \\
		$\tau_1, \tau_2, \cdots$ & Indices for the subsequence of times at which the patient was assigned standard-of-care (SOC) \\
		$Z_t$ & Predictors in the standard monitoring setting\\
		$\hat{f}_t$ & The ML algorithm at time $t$\\
		$\kappa^{\rel} \in (1,K)$ & Position of changepoint in relative time\\
		$\kappa$ & Changepoint in absolute time, equal to $\lfloor m \kappa^{\rel} \rfloor $\\
		$\theta$ & Parameter indexing the pre-change distribution\\
		$\delta$ & Parameter indexing the structural change\\
		$m$ & Size of dataset needed for initialization of monitoring procedures (also known as non-contaminated data)\\
		$m + 1, \cdots, mK$ & The time period for monitoring structural change\\
		$C_m(t)$ & The chart statistic of a monitoring procedure at time $t$\\
		$h_m(t)$ & The control limit of a monitoring procedure at time $t$\\
		$\hat{T}_m$ & Alarm time of a monitoring procedure, i.e. when the chart statistic first exceeds the control limit
	\end{tabular}
\end{table}

\section{Example satisfying the time-constant selection bias assumption}
\begin{example}
	Consider the bottom single world intervention graph (SWIG) in Figure~\ref{fig:dags}, where $U_t$ is an unmeasured confounder.
	By the rules of D-separation, we have that \eqref{eq:indpt} holds.
	Suppose the distribution $(U_t, A_t')$ is constant over time; as such, we will drop the time indices when denoting their marginal and conditional distributions.
	For the conditional risk model, assume there is no-additive interaction with respect to time and $U_t$, i.e.
	$$
	\E\left[Y_t(0)\mid X_t = x, U_t = u\right] = g_0(x, u; \theta_0) + g_1(x; \delta_1)\mathbbm{1}\{t > \kappa \}.
	$$
	
	Because $A_{t} = 0$ implies that $A_t' = 0$, we have $U_t\perp X_t \mid f_t(X_t)=q, A_t = 0$.
	Then for all times $t$ and $q \in \mathcal{Q}$, we have that
	\begin{align}
	& \E\left[Y_{t}(0)\mid\hat{f}_t(X_t)=q\right]-\E\left[Y_{t}(0)\mid\hat{f}_t(X_t)=q,A_t=0\right]\\
	= & \int
	\left(g_0(x_t, u; \theta_0) + g_1(x_t; \delta_1)\mathbbm{1}\{t>\kappa\} \right)
	p(x_t \mid \hat{f}_t(x_t) = q)
	\left[p(u) - p(u \mid a' = 0) \right] dx_t du \\
	= & \int
	g_0(x_t, u; \theta_0)
	p(x_t \mid \hat{f}_t(x_t) = q)
	\left[p(u) - p(u \mid a' = 0) \right] dx_t du.
	\label{eq:ex_conf}
	\end{align}
	There are various conditions under which \eqref{eq:ex_conf} is time-constant.
	One requirement is that $g_0(x,u; \theta_0)$ is additive, in that $g_0(x,u; \theta_0) = g_{0,0}(x; \theta_0) +  g_{0,1}(u; \theta_0)$.
	Alternatively, we require the ML algorithm to be locked ($\hat{f}_t = \hat{f}$ for all $t$) and the distribution of $X_t$ to not vary over time.
	
	\label{example:equiconf}
\end{example}

\section{Proofs for the score-based CUSUM}
Let $\mathcal{Z}$ and $\mathcal{Y}$ be the domains for the predictors and outcomes. 
Let $(\theta_{0},\delta_{0})$ and $(\theta_{0},\delta_{1})$ parameterize the pre-change and post-change distribution, where $\delta_{0} = 0$ and $\delta_1 \ne 0$.
We assume that $p(y|z; \theta, \delta)$ is 3-times continuously differentiable with respect to $(\theta, \delta)$.
For convenience, denote
\begin{align*}
\Lambda_{m}(i) & = \mathbb{E}\left[-\sum_{j=1}^{i}
\nabla_{\theta}^{2}\log p\left(Y_{j}\mid Z_{j};\theta_{0},\delta_0\right)
\right]\\
V_m(i) & = \mathbb{E}\left[\nabla_{\theta}\nabla_{\delta}\log p\left(Y_{i}\mid Z_{i};\theta_0,\delta_0\right)
\middle |
Z_{i} \right].
\end{align*}
For $v\in[1,K]$, define the limit of the MLEs as $\bar{\hat{\theta}}(v) = \lim_{m \rightarrow \infty } \hat{\theta}_{m,\lfloor mv \rfloor }(v)$.

We use the symbol $\Rightarrow$ to mean weak convergence in the space under consideration.
Throughout, we will use $c$ (sometimes with subscripts) to denote constants, which may vary across contexts.
When we write $Z_m \le_{p} c$, this means that asymptotically as $m\rightarrow \infty$, the random variable $Z_m$ is bounded by some constant $c$ with probability 1.

\subsection{Asymptotics under the null}
\label{sec:null_asym}
Here we prove asymptotic convergence of the chart statistic under the null.
In addition to the assumptions listed in the main manuscript, we will require the second and third derivatives of the likelihood to be bounded as follows.
\begin{assumption}
	There exist constants $c_1, c_2>0$ and some neighborhood $B(\theta_{0},r)$
	centered at $\theta_{0}$ with radius $r>0$ such that 
	\begin{align}
	\sup_{\tilde{\theta}\in B(\theta_{0},r)}\sup_{y\in\mathcal{Y},z\in\mathcal{Z}}\left\Vert \nabla_{\theta}^{2}\nabla_{\delta}\log p\left(y\mid z;\tilde{\theta},\delta_{0}\right)\right\Vert _{\infty} & \le c_1
	\label{eq:bounded_deriv3}\\
	\sup_{y\in\mathcal{Y},z\in\mathcal{Z}}\left\Vert \nabla_{\theta}\nabla_{\delta}\log p\left(y\mid z;\theta_{0},\delta_{0}\right)\right\Vert _{\infty} & \le_{p}c_2.\label{eq:bounded_deriv2}
	\end{align}
	\label{assum:bounded}
\end{assumption}

\begin{lemma}
	Suppose Assumptions~\ref{assume:gauss}, \ref{assume:remain}, \ref{assume:matrices}, and \ref{assum:bounded} hold.
	Define 
	\[
	\tilde{\phi}_{m}(t_{1},t_{2})=
	\sum_{i=t_{1}}^{t_{2}}\nabla_{\delta}\log p\left(Y_{i}\mid Z_{i};\theta_{0},\delta_{0}\right)
	+\sum_{i=t_{1}}^{t_{2}}V_{m}(i)\Lambda_{m}^{-1}(i-1)\sum_{j=1}^{i-1}\nabla_{\theta}\log p\left(Y_{j}\mid Z_{j};\theta_{0},\delta_{0}\right).
	\]
	Under the null, we have
	\[
	\max_{m<t_{1},t_{2}\le mK}
	\frac{1}{\sqrt{m}}
	\left\Vert \psi_{m}^{(\plugin)}(t_{1},t_{2})-\tilde{\phi}_{m}(t_{1},t_{2})\right\Vert_2 =o_{p}(1).
	\]
	\label{lemma:remain}
\end{lemma}

\begin{proof}
	Consider the decomposition 
	\begin{align*}
	\frac{1}{\sqrt{m}}
	\left(\psi_{m}^{(\plugin)}(t_{1},t_{2})-\tilde{\phi}_{m}(t_{1},t_{2}) \right)=
	& R_{m}^{(1)}(t_{1},t_{2})+R_{m}^{(2)}(t_{1},t_{2})+R_{m}^{(3)}(t_{1},t_{2})
	\end{align*}
	where
	\begin{align*}
	R_{m}^{(1)}(t_{1},t_{2}) & =\frac{1}{\sqrt{m}}\sum_{i=t_{1}}^{t_{2}}\nabla_{\delta}\log p\left(Y_{i}\mid Z_{i};\hat{\theta}_{m,i-1},\delta_0\right)-\nabla_{\delta}\log p\left(Y_{i}\mid Z_{i};\theta_{0},\delta_{0}\right)-\nabla_{\theta}\nabla_{\delta}\log p\left(Y_{i}\mid Z_{i};\theta_0,\delta_0\right) \left(\hat{\theta}_{m,i-1}-\theta_{0}\right)\\
	R_{m}^{(2)}(t_{1},t_{2}) & =\frac{1}{\sqrt{m}}\sum_{i=t_{1}}^{t_{2}}V_{m}(i)\left[\hat{\theta}_{m,i-1}-\theta_{0}-\Lambda_{m}^{-1}(i-1)\sum_{j=1}^{i-1}\nabla_{\theta}\log p\left(Y_{j}\mid Z_{j};\theta_{0},\delta_{0}\right)\right]\\
	R_{m}^{(3)}(t_{1},t_{2}) & =\frac{1}{\sqrt{m}}\sum_{i=t_{1}}^{t_{2}}\left[\nabla_{\theta}\nabla_{\delta}\log p\left(Y_{i}\mid Z_{i};\theta_{0},\delta_{0}\right)-V_{m}(i)\right]\left(\hat{\theta}_{m,i-1}-\theta_{0}\right).
	\end{align*}
	We will prove that each term in this decomposition is negligible,
	i.e. 
	\begin{equation}
	\max_{m<t_{1}<t_{2}\le mK}
	\left \|R_{m}^{(j)}(t_{1},t_{2})\right \|_2=o_{p}(1)
	\quad\forall j=1,2,3.
	\label{eq:remainder_decomp}
	\end{equation}
	
	\paragraph{First remainder term.} For any $\epsilon>0$, we have
	that
	\begin{align}
	\begin{split}
	\Pr\left(
	\max_{m<t_{1}<t_{2}\le mK}
	\left \|R_{m}^{(1)}(t_{1},t_{2}) \right \|_2 > \epsilon\right) \le
	& \Pr\left(\max_{m<i\le mK}
	\left\Vert \hat{\theta}_{m,i}-\theta_{0}\right\Vert _{2}>c\frac{\log m}{\sqrt{m}}\right) \\
	& +\Pr\left(\max_{m<t_{1}<t_{2}\le mK}
	\left \|R_{m}^{(1)}(t_{1},t_{2})\right \|_2>\epsilon,
	\max_{m<i\le mK}\left\Vert \hat{\theta}_{m,i}-\theta_{0}\right\Vert _{2}\le c\frac{\log m}{\sqrt{m}}\right).
	\label{eq:r1}
	\end{split}
	\end{align}
	The first summand on the RHS of \eqref{eq:r1} goes to zero because Assumptions \ref{assume:gauss}, \ref{assume:remain}, and \ref{assume:matrices} imply 
	\begin{equation}
	\max_{m<i\le mK}\left\Vert \hat{\theta}_{m,i}-\theta_{0}\right\Vert _{2}=o_{p}\left(\frac{\log m}{\sqrt{m}}\right).
	\label{eq:mle_prob}
	\end{equation}
	To bound the second summand, we have by Taylor's theorem and
	Assumption~\ref{assum:bounded} that for sufficiently large
	$m$
	\begin{align*}
	\left\Vert R_{m}^{(1)}(t_{1},t_{2})\right\Vert_2  & \le\frac{1}{2\sqrt{m}}\sum_{i=t_{1}}^{t_{2}}\left(\max_{\tilde{\theta}\in B(\theta_{0},r)}\max_{z\in\mathcal{Z},y\in\mathcal{Y}}\left\Vert \nabla_{\theta}^{2}\nabla_{\delta}\log p\left(y\mid z;\tilde{\theta},\delta_{0}\right)\right\Vert _{\infty}\right)\left\Vert \hat{\theta}_{m,i-1}-\theta_{0}\right\Vert _{2}^{2}\\
	& \le\frac{c_1}{2\sqrt{m}}\sum_{i=t_{1}}^{t_{2}}\left\Vert \hat{\theta}_{m,i-1}-\theta_{0}\right\Vert _{2}^{2}\\
	& \le \frac{c_1}{2}\sqrt{m}\left(K-1\right)\max_{m<i\le mK}\left\Vert \hat{\theta}_{m,i}-\theta_{0}\right\Vert _{2}^{2}\\
	& =o_{p}\left(\left(K-1\right)\left(\log m\right)^{2}/\sqrt{m}\right)
	\end{align*}
	for all $(t_{1},t_{2})$ where $m<t_{1}\le t_{2}\le mK$. So (\ref{eq:remainder_decomp})
	holds for $j=1$.
	
	\paragraph{Second remainder term.}
	By Assumption~\ref{assum:bounded} and the Cauchy-Schwarz inequality, we have that
	\begin{align*}
	\left\Vert R_{m}^{(2)}(t_{1},t_{2})\right\Vert_2  & \le\frac{c_2}{\sqrt{m}}\sum_{i=t_{1}}^{t_{2}}\max_{m<i\le mK}\left\Vert \hat{\theta}_{m,i-1}-\theta_{0}-\Lambda_{m}^{-1}(i-1)\sum_{j=1}^{i-1}\nabla_{\theta}\log p\left(Y_{j}\mid Z_{j};\theta_{0},\delta_{0}\right)\right\Vert _{2}\\
	& \le c\left(K-1\right)\sqrt{m}\max_{m<i\le mK}\left\Vert \hat{\theta}_{m,i}-\theta_{0}-\Lambda_{m}^{-1}(i)\sum_{j=1}^{i}\nabla_{\theta}\log p\left(Y_{j}\mid Z_{j};\theta_{0},\delta_{0}\right)\right\Vert _{2}.
	\end{align*}
	Then by Assumption~\ref{assume:remain}, this term is $o_{p}\left(1\right)$.
	So (\ref{eq:remainder_decomp}) holds for $j=2$.
	
	\paragraph{Third remainder term.} For any $\epsilon>0$, we have that 
	\begin{align}
	\begin{split}
	& \Pr\left(\max_{t_{1},t_{2}}
	\left \|
	\sum_{i=t_{1}}^{t_{2}}\left(\nabla_{\theta}\nabla_{\delta}\log p\left(Y_{i}\mid Z_{i};\theta_{0},\delta_{0}\right)-V_{m}(i)\right)\left(\hat{\theta}_{m,i-1}-\theta_{0}\right)
	\right \|_2
	\ge\epsilon\sqrt{m}\right)\\
	\le & \Pr\left(\max_{m<i\le mK}\left\Vert \hat{\theta}_{m,i}-\theta_{0}\right\Vert _{2}>c\frac{\log m}{\sqrt{m}}\right)\\
	& +\Pr\left(
	\max_{t_{1},t_{2}}
	\left \|
	\sum_{i=t_{1}}^{t_{2}}\left(\nabla_{\theta}\nabla_{\delta}\log p\left(Y_{i}\mid Z_{i};\theta_{0},\delta_{0}\right)-V_{m}(i)\right)\left(\hat{\theta}_{m,i-1}-\theta_{0}\right)
	\right \|_2
	\ge\epsilon\sqrt{m},\max_{m<i\le mK}\left\Vert \hat{\theta}_{m,i}-\theta_{0}\right\Vert _{2}\le c\frac{\log m}{\sqrt{m}}\right).
	\label{eq:decomp3}
	\end{split}
	\end{align}
	Per \eqref{eq:mle_prob}, the first summand on the RHS of \eqref{eq:decomp3} goes to zero.
	The second summand is bounded by 
	\begin{equation}
	\Pr\left(
	\max_{t_{1},t_{2}}
	\left \|
	\sum_{i=t_{1}}^{t_{2}}\left(\nabla_{\theta}\nabla_{\delta}\log p\left(Y_{i}\mid Z_{i};\theta_{0},\delta_{0}\right)-V_{m}(i)\right)\left(\hat{\theta}_{m,i-1}-\theta_{0}\right)\mathbbm{1}\left\{ \|\hat{\theta}_{m,i-1}-\theta_{0}\|_2 \le c\frac{\log m}{\sqrt{m}}\right\}
	\right \|_2
	\ge\epsilon\sqrt{m}\right).
	\label{eq:martingale_pr}
	\end{equation}
	Because the outcome $Y_{i}$ is conditionally independent of past
	data given $Z_{i}$, the elements in this summation form a martingale
	difference sequence, i.e. 
	\[
	\E\left[G_{m}(i)\left(\hat{\theta}_{m,i-1}-\theta_{0}\right)\mathbbm{1}\left\{ \left \|\hat{\theta}_{m,i-1}-\theta_{0}\right \|_2 \le c\frac{\log m}{\sqrt{m}}\right\} 
	\middle \vert
	\mathcal{F}_{i}\right]=0
	\]
	where we use the notational shorthand
	\[
	G_{m}(i)=\nabla_{\theta}\nabla_{\delta}\log p\left(Y_{i}\mid Z_{i};\theta_{0},\delta_{0}\right)-V_{m}(i)
	\]
	and
	
	\[
	\mathcal{F}_{i}=\left(Z_{1},Y_{1},\cdots,Z_{i-1},Y_{i-1},Z_{i}\right).
	\]
	Moreover, by Assumption~\ref{assum:bounded}, $G_{m}(i)$ is sub-Gaussian.
	That is, there is some $\sigma^{2}>0$ such that for all $\lambda>0$,
	we have for all unit vectors $u,v$ that
	\begin{align*}
	\E\left[\exp\left(\lambda v^{\top}G_{m}(i)u\right)
	\middle |
	\mathcal{F}_{i-1}\right] & \le \E\left[\exp\left(\lambda^{2}\sigma^{2}\lambda_{\max}\left(G_{m}(i)\right)\right)
	\middle |
	\mathcal{F}_{i-1}\right].
	\end{align*}
	By the law of total expectations, we then have for any unit vector
	$v$ that
	\begin{align*}
	& \E\left[\exp\left(\lambda v^{\top}\sum_{i=t_{1}}^{t_{2}}G_{m}(i)\left(\hat{\theta}_{m,i-1}-\theta_{0}\right)
	\mathbbm{1}\left\{ \left \|\hat{\theta}_{m,i-1}-\theta_{0}\right \|_2 \le c\frac{\log m}{\sqrt{m}}\right\}
	\right)\right]\\
	= & \E\Biggr[
	\E\left[\exp\left(\lambda v^{\top}G_{m}(t_{2})\left(\hat{\theta}_{m,t_{2}-1}-\theta_{0}\right)
	\mathbbm{1}\left\{ \left \|\hat{\theta}_{m,t_2-1}-\theta_{0}\right \|_2 \le c\frac{\log m}{\sqrt{m}}\right\}
	\right)\middle | 
	\mathcal{F}_{t_{2}-1}\right]
	\\
	& \qquad \times \exp\left(\lambda v^{\top}\sum_{i=t_{1}}^{t_{2}-1}G_{m}(i)\left(\hat{\theta}_{m,i-1}-\theta_{0}\right)
	\mathbbm{1}\left\{ \left\|\hat{\theta}_{m,i-1}-\theta_{0}\right \|_2 \le c\frac{\log m}{\sqrt{m}}\right\} 
	\right)
	\Biggr]
	\\
	\le & \exp\left(\lambda^{2}c^{2}\sigma^{2}\frac{\left(\log m\right)^{2}}{m}\right)
	\E\left[\exp\left(\lambda v^{\top}\sum_{i=t_{1}}^{t_{2}-1}G_{m}(i)\left(\hat{\theta}_{m,i-1}-\theta_{0}\right)
	\mathbbm{1}\left\{ \left\|\hat{\theta}_{m,i-1}-\theta_{0}\right \|_2 \le c\frac{\log m}{\sqrt{m}}\right\}
	\right)\right]\\
	\le & \exp\left(\lambda^{2}c^{2}\sigma^{2}\left(K-1\right)\left(\log m\right)^{2}\right).
	\end{align*}
	Using the Chernoff bound, we have that (\ref{eq:martingale_pr}) is
	bounded by 
	\begin{align*}
	& 
	\sum_{m<t_{1}\le t_{2}\le mK}\Pr\left(
	\left \|
	\sum_{i=t_{1}}^{t_{2}}
	\left(\nabla_{\theta}\nabla_{\delta}\log p\left(Y_{i}\mid Z_{i};\theta_{0},\delta_{0}\right)-V_{m}(i)\right)\left(\hat{\theta}_{m,i-1}-\theta_{0}\right)\mathbbm{1}\left\{ \left\|\hat{\theta}_{m,i-1}-\theta_{0}\right \|_2 \le c\frac{\log m}{\sqrt{m}}\right\} 
	\right \|_2
	\ge\epsilon\sqrt{m}\right)\\
	\le & m^{2}\left(K-1\right)^{2}\exp\left(\lambda^{2}c^{2}\sigma^{2}\left(K-1\right)\left(\log m\right)^{2}-\epsilon^{2}m\right).
	\end{align*}
	This converges to zero as $m\rightarrow\infty$, so (\ref{eq:remainder_decomp})
	holds for $j=3$.
\end{proof}

Using Lemma~\ref{lemma:remain} above, we are now ready to prove Theorem~\ref{thrm:score_conv}.

\begin{proof}[Proof of Theorem~\ref{thrm:score_conv}]
	Consider the decomposition
	\begin{align}
	\tilde{\phi}_{m}(t_{1},t_{2}) &  =
	\phi_{m}(t_{1},t_{2}) +R_{m}^{(1)}(t_{1},t_{2})+R_{m}^{(2)}(t_{1},t_{2})
	\label{eq:decomp_phi}
	\end{align}
	with remainder terms defined as
	\begin{align*}
	R_{m}^{(1)}(t_{1},t_{2})= & \frac{1}{\sqrt{m}}\sum_{i=t_{1}}^{t_{2}}\left(V_{m}(i)
	-\bar{V}_{0}\left(i/m\right)\right)\Lambda_{m}^{-1}(i-1)\sum_{j=1}^{i-1}\nabla_{\theta}\log p\left(Y_{j}\mid Z_{j};\theta_{0}, \delta_0\right)\\
	R_{m}^{(2)}(t_{1},t_{2})= & \frac{1}{\sqrt{m}}\sum_{i=t_{1}}^{t_{2}}\bar{V}_{0}\left(i/m\right)\left(\Lambda_{m}^{-1}(i-1)-\frac{1}{m}\Lambda_{0}^{-1}\left(\frac{i-1}{m}\right)\right)\sum_{j=1}^{i-1}\nabla_{\theta}\log p\left(Y_{j}\mid Z_{j};\theta_{0}, \delta_0\right).
	\end{align*}
	We will first show the remainder terms are negligible, i.e.
	
	\begin{equation}
	\max_{m<t_{1}<t_{2}\le mK}\|R_{m}^{(j)}(t_{1},t_{2})\|=o_{p}(1)\quad\forall j=1,2.\label{eq:remainder_decomp2}
	\end{equation}
	
	\paragraph{First remainder term.} To bound the first remainder,
	note that the partial sums form a martingale due to Assumption~\ref{assume:matrices},
	i.e.
	\[
	\E\left[
	\left(V_{m}(i)-\bar{V}_{0}\left(i/m\right)\right)\Lambda_{m}^{-1}(i-1)\sum_{j=1}^{i-1}\nabla_{\theta}\log p\left(Y_{j}\mid Z_{j};\theta_{0}, \delta_0 \right)
	\middle | 
	\mathcal{F}_{i}
	\right]
	=0.
	\]
	Moreover,
	$\left(V_{m}(i)-\bar{V}_{0}(i/m)\right)\Lambda_{m}^{-1}(i)$ is sub-Gaussian
	and 
	\begin{align*}
	\max_{m<i\le mK}\left\Vert \sum_{j=1}^{i-1}\nabla_{\theta}\log p\left(Y_{j}\mid Z_{j};\theta_{0}, \delta_0\right)\right\Vert _{2} & =o_{p}\left(\frac{\log m}{\sqrt{m}}\right)
	\end{align*}
	per Assumptions~\ref{assume:gauss} and \ref{assume:matrices}.
	As such, we can use a similar martingale argument as the previous lemma to prove that \eqref{eq:remainder_decomp2} is satisfied for $j = 1$.

	\paragraph{Second remainder term.} By the Cauchy-Schwarz inequality,
	we have that
	\begin{align*}
	& \frac{1}{\sqrt{m}}\sum_{i=t_{1}}^{t_{2}}\bar{V}_{0}\left(i/m\right)\left(\Lambda_{m}^{-1}(i-1)-\frac{1}{m}\Lambda_{0}^{-1}\left(\frac{i-1}{m}\right)\right)\sum_{j=1}^{i-1}\nabla_{\theta}\log p\left(Y_{j}\mid Z_{j};\theta_{0}, \delta_0\right)\\
	& \le c\sum_{i=t_{1}}^{t_{2}}\left\Vert \Lambda_{m}^{-1}(i)-\frac{1}{m}\Lambda_{0}^{-1}(i/m)\right\Vert _{2}\left\Vert \frac{1}{\sqrt{m}}\sum_{j=1}^{i-1}\nabla_{\theta}\log p\left(Y_{j}\mid Z_{j};\theta_{0}, \delta_0\right)\right\Vert _{2}\\
	& \le c\left(K-1\right)\left(\max_{i=m+1,\cdots,mK}\left\Vert m\Lambda_{m}^{-1}(i)-\Lambda_{0}^{-1}(i/m)\right\Vert _{2}\right)\left(\max_{m<i\le mK}\left\Vert \frac{1}{\sqrt{m}}\sum_{j=1}^{i-1}\nabla_{\theta}\log p\left(Y_{j}\mid Z_{j};\theta_{0}, \delta_0\right)\right\Vert _{2}\right)
	\end{align*}
	By Assumptions~\ref{assume:gauss} and \ref{assume:matrices}, it follows that \eqref{eq:remainder_decomp2} for $j=2$.
	
	In addition, by Assumption~\ref{assume:gauss}, we have that
	\[
	\left\{
	\nu
	\mapsto
	\frac{1}{\sqrt{m}}\sum_{j=1}^{\lfloor m \nu \rfloor }
	\left(\begin{array}{c}
	\nabla_{\theta}\log p\left(Y_{j}\mid Z_{j};\theta_{0},\delta_0\right)\\
	V_{0}(\nu ) \Lambda_{0}^{-1}(\nu )
	\nabla_{\delta}\log p\left(Y_{j}\mid Z_{j};\theta_{0},\delta_0\right)
	\end{array}\right)
	\right\}
	\Rightarrow
	\left\{
	\nu
	\mapsto
	\left(\begin{array}{c}
	U_{\theta}(i/m)\\
	V_{0}(\nu ) \Lambda_{0}^{-1}(\nu ) U_{\delta}(i/m)
	\end{array}\right)
	\right\}.
	\]
	So by Slutsky's theorem and the continuous mapping theorem, we have weak convergence of the process $\phi_m$ with respect to the space of bounded functions $f:\Delta \mapsto \mathbb{R}^{d}$ as follows
	\begin{align}
	\left\{ 
	(\nu_{1},\nu_{2})
	\mapsto
	\phi_{m}(\nu_{1},\nu_{2})
	\right\}_{(\nu_1, \nu_2) \in \Delta}
	& \Rightarrow
	\left\{
	(\nu_{1},\nu_{2})
	\mapsto U_{\delta}(\nu_{2})-U_{\delta}(\nu_{1})+\int_{\nu_{1}}^{\nu_{2}}\bar{V}_{0}(v)\Lambda_{0}^{-1}(v)U_{\theta}(v)dv
	\right\}_{(\nu_1, \nu_2) \in \Delta}.
	\label{eq:phi_asymp}
	\end{align}
	Combining this result with Lemma~\ref{lemma:remain} and \eqref{eq:remainder_decomp2}, the process $\psi_m^{(\plugin)}$ converges weakly to the same limit as $\phi_m$.
\end{proof}

\subsection{Asymptotics under the alternative}

Suppose there is some $K'\in(\kappa,K]$ that satisfies the following assumptions.
Assumptions~\ref{assum:gauss_alt}, \ref{assum:consist_alt}, and \ref{assum:bounded_alt} can be viewed as analogous but simplified versions of the assumptions in Section~\ref{sec:null_asym}.
Assumption~\ref{assum:score_nonzero} assumes the cumulative score process is characterized by some non-zero drift under the alternative for some time period, even when we continually update the plugin estimator $\hat{\theta}_{m,i}$.
For example, this is likely to hold for values of $K'$ that are slightly larger than $\kappa$, since the plugin estimators up to time $\lfloor m K' \rfloor$ will not have strayed too far from its actual value of $\theta_0$ prior to the changepoint.

\begin{assumption}
	Under the alternative, suppose that 
	\begin{equation*}
	\frac{1}{\sqrt{m}}\sum_{j=\lfloor m\kappa\rfloor}^{\lfloor m K' \rfloor }
	\nabla_{\delta}\log p\left(Y_{j}\mid Z_{j};\bar{\hat{\theta}}\left(\frac{j-1}{m}\right),\delta_{0}\right)
	-\mathbb{E}_{\theta_{0},\delta_{1}\mathbbm{1}\left\{ i\ge m\kappa\right\} }\left[
	\nabla_{\delta}\log p\left(Y_{j}\mid Z_{j};\bar{\hat{\theta}}\left(\frac{j-1}{m}\right),\delta_{0}\right)
	\middle |
	Z_j
	\right]
	=O_{P}(1).
	\end{equation*}
	\label{assum:gauss_alt}
\end{assumption}

\begin{assumption}
	Under the alternative, suppose that 
	\begin{equation*}
	\max_{i=\lfloor m\kappa\rfloor,\cdots,\lfloor m K' \rfloor}\left\Vert \sqrt{m}\left(\hat{\theta}_{m,i}-\bar{\hat{\theta}}(i/m)\right)\right\Vert _{2}=O_{p}(1).\label{eq:mle_converge}
	\end{equation*}
	\label{assum:consist_alt}
\end{assumption}

\begin{assumption}
	There is some $c>0$ and neighborhood $B$ that includes the set $\left\{ \bar{\hat{\theta}}\left(v\right):v\in[\tau,K']\right\} $
	with nonzero radius such that 
	\begin{equation*}
	\sup_{\theta\in B}\sup_{y\in\mathcal{Y},z\in\mathcal{Z}}
	\left\Vert
	\nabla_{\theta}\nabla_{\delta}\log p\left(y\mid z;\theta,\delta_0\right)
	\right\Vert _{\infty}
	\le_{p}c.
	\end{equation*}
	\label{assum:bounded_alt}
\end{assumption}

\begin{assumption}
	There is some $c>0$ such that
	\begin{equation*}
	\lim_{m \rightarrow \infty}
	\left\Vert
	\frac{1}{m}
	\sum_{j = \lfloor m \kappa \rfloor }^{\lfloor m K' \rfloor }
	\mathbb{E}_{\theta_{0},\delta_{1}}\left[
	\nabla_{\delta}\log p\left(Y_j \mid Z_j;\bar{\hat{\theta}}(t),\delta_{0}\right)
	\right]
	\right\Vert_2
	\ge c.
	\end{equation*}
	\label{assum:score_nonzero}
\end{assumption}

\begin{theorem}
	Suppose Assumptions~\ref{assum:gauss_alt} to \ref{assum:score_nonzero} hold. Then under the alternative hypothesis, we have
	\[
	\lim_{m\rightarrow\infty}\Pr\left(\exists t \in \{m+1,\cdots,mK\} \text{ such that } C_{m}^{(\plugin)}(t)>h_{m}(t)\right)=1.
	\]
	\label{thrm:alt_consist}
\end{theorem}

\begin{proof}
	By the definition of $C_{m}^{(\plugin)}(t)$, suffices to prove that
	\[
	\frac{1}{\sqrt{m}}\psi_{m}^{(\plugin)}(\lfloor m \kappa \rfloor,\lfloor m K' \rfloor)=\frac{1}{\sqrt{m}}\sum_{i=\lfloor m \kappa \rfloor}^{\lfloor m K' \rfloor}\nabla_{\delta}\log p\left(Y_{i}\mid Z_{i};\hat{\theta}_{m,i-1},\delta_{0}\right)
	\]
	goes to infinity.
	Consider the following decomposition
	\begin{align}
	\begin{split}
	\frac{1}{\sqrt{m}}\psi_{m}^{(\plugin)}(\lfloor m \kappa \rfloor,\lfloor m K' \rfloor)=
	& \frac{1}{\sqrt{m}}\sum_{i=\lfloor m \kappa \rfloor}^{\lfloor m K' \rfloor}
	\E_{\theta_{0},\delta_{1}}\left[\nabla_{\delta}\log p\left(Y_{i}\mid Z_{i};\bar{\hat{\theta}}\left(\frac{i-1}{m}\right ),\delta_{0}\right)\right] \\
	& +\frac{1}{\sqrt{m}}\sum_{i=\lfloor m \kappa \rfloor}^{\lfloor m K \rfloor}
	\left\{
	\nabla_{\delta}\log p\left(Y_{i}\mid Z_{i};\bar{\hat{\theta}}\left(\frac{i-1}{m} \right ),\delta_{0}\right)
	-\mathbb{E}_{\theta_{0},\delta_{1}}\left[\nabla_{\delta}\log p\left(Y_{i}\mid Z_{i};\bar{\hat{\theta}}\left(\frac{i-1}{m}\right),\delta_{0}\right)\right]
	\right\}
	\\
	& + R_m
	\end{split}
	\label{eq:plugin_alt}
	\end{align}
	The first term on the right hand side must diverge to infinity by Assumption~\ref{assum:score_nonzero}.
	Per Assumption~\ref{assum:gauss_alt}, the second term is $O_{p}(1)$.
	Per Assumption~\ref{assum:consist_alt} and Taylor's theorem, the remainder $R_m$ satisfies
	\begin{align}
	\lim_{m\rightarrow \infty} \left \|R_m \right \|_2
	\le 
	\lim_{m\rightarrow \infty}
	\frac{c_1}{\sqrt{m}}
	\sup_{\{\tilde{\theta}_{m,i-1}: i= \lfloor m\kappa \rfloor, \cdots \lfloor mK'\rfloor \} \in B}
	\left \|
	\sum_{i=\lfloor m \kappa \rfloor}^{\lfloor m K' \rfloor}\nabla_{\theta}\nabla_{\delta}\log p\left(Y_{i}\mid Z_{i};
	\tilde{\theta}_{m,i-1},\delta_{0}\right)\left(\hat{\theta}_{m,i-1}-\bar{\hat{\theta}}\left(\frac{i-1}{m}\right )\right)
	\right \|_2.
	\end{align}
	Combined with Assumption \ref{assum:bounded_alt}, we have that
	\begin{align}
	\lim_{m\rightarrow \infty} \left \|R_m \right \|_2
	\le 
	\lim_{m\rightarrow \infty} 
	c_2\sqrt{m}
	\max_{i = \lfloor m \kappa \rfloor, \cdots, \lfloor m K' \rfloor }
	\left\Vert
	\hat{\theta}_{m,i-1}-\bar{\hat{\theta}}\left (\frac{i-1}{m} \right)
	\right\Vert_2
	= O_{p}(1).
	\end{align}
	As such, the right hand side of \eqref{eq:plugin_alt} diverges to infinity, which means its left hand side must also diverge to infinity.
	Thus we have our desired result.
\end{proof}

\section{Implementation details for the score-based CUSUM}
\label{sec:implement}
There are a number of implementation decisions to make.
First, the number of sequences $B$ should be set to a value large enough such that the estimated DCLs converge.
In our simulations, we chose $B$ so that the chart statistic for five or more bootstrapped sequences exceeded the DCL at each time step.
Next, one can maximize statistical power by tuning the shape of the alpha-spending function.
Here we simply use a linear alpha-spending function, but future work may explore nonlinear functions instead.
Finally, our theoretical results allow for monitoring in a fully sequential manner or in batches of observations.
We found that batching had a negligible impact on detection delay and improved both computational efficiency and convergence to the asymptotic distribution.
As such, we recommend setting the batch size so a significant change is unlikely to occur within a batch.
In our experiments, the batch size is set to 10.

\section{Bayesian changepoint monitoring}
In this section, we briefly review Bayesian changepoint monitoring and discuss its implementation.

To determine if a structural change has occurred, the chart statistic in Bayesian monitoring is the posterior probability of there having been a change, i.e. $\hat{C}^{\bayes}(t)=\Pr\left(\kappa \le t\mid Y_1, \cdots, Y_{t}, Z_1, \cdots, Z_t\right)$.
The operating characteristics of a Bayesian procedure is defined by its probability of firing a false alarm $\Pr\left(\kappa>\hat{T}^{\bayes}\right)$, where $\hat{T}^{\bayes}$ is the alarm time and the probability is marginalized over the prior.
One can show that the static control limit of $1 - \alpha$ indeed controls the false alarm probability at level $\alpha$, albeit conservatively (see \citep{Tartakovsky2014-yq}, page 318).

The performance of the Bayesian monitoring procedue is sensitive to the choice of the prior.
Given the minimax optimality of the Shiryaev-Roberts procedure \citep{Shiryaev1963-xq}, we use a modified geometric distribution for the prior of $\kappa$, in which
\begin{align}
\pi(\kappa = t) \propto p (1 - p)^{t - 1} \quad \forall t = m+1,\cdots,mK
\end{align}
with $p=1/mK$ and the probability of there being no changepoint is set to $0.5$.
We assume a normal prior for $\theta$ with the mean and covariance matrix set to the results from maximum likelihood estimation on the non-contaminated data.
We also place a normal prior for $\delta$ with mean zero and a diagonal covariance matrix.
In the simulations, we set the diagonal matrix so that the mean norm of $\delta$ in the prior is close to that of the actual shift.
In practice, such information is not known and one must rely on prior knowledge.

In this paper, we use Hamiltonian Monte Carlo (HMC) implemented using Stan \citep{Carpenter2017-ze}.
While this was sufficient for running experiments, there are a number of caveats when running HMC.
First, HMC is not recommended for long-term monitoring, because it runs posterior inference from scratch for every new observation.
A potential solution is to perform Bayesian filtering with a Laplace approximation \citep{McCormick2012-sg, Feng2022-mk}, but existing methods have yet to be specialized for detecting a single changepoint.
Second, HMC requires the model to be fully differentiable, so it is not guaranteed to provide valid inference for the risk shift model \eqref{eq:mdl_tc}.

\section{Simulation details}

For each simulation, we generate a $p'$-random vector $X_t$ and random variables $\tilde{X}_t$ and $U_t$.
All these variables are drawn independently from the uniform distribution from -1 to 1.
The outcome $Y_t(0)$ was generated using either \eqref{eq:mdl_de} or \eqref{eq:mdl_tc} with $Z_t = (X_t, \tilde{X}_t, U_t, 1)$.
Treatments were assigned using one of two models.
The first is a logistic regression model with $(\hat{f}_t(X_t), X_t, \tilde{X}_t, U_t)$ with coefficients and intercept denoted by $\gamma_{\LR}$.
The second sets treatment to $A_t = \max(A_{t}^{(1)}, A_{t}^{(2)})$, where $A_{t}^{(1)}$ and $A_{t}^{(2)}$ are generated using logistic regression models with inputs $(\hat{f}_t(X_t), X_t, \tilde{X}_t, U_t)$, with coefficients and intercept denoted by $\gamma^{(1)}$ and $\gamma^{(2)}$, respectively.
The model parameters used to generate outcomes and treatment assignments are given in Table~\ref{table:sim_deets}.

\begin{landscape}
	\begin{table}
		\centering
		\begin{tabular}{c|p{2in}|p{3in}|p{2.6in}}
			Section & Experiment(s) & Outcome model & Treatment model  \\
			\toprule
			\ref{sec:sim_false} & CE with respect to $\hat{f}_t(X_t)$  &
			\eqref{eq:mdl_de} with
			$\theta = (2,1,1,1,\vec{0}_4,0, 0,0)$ and $\delta = \vec{0}$ &
			$\gamma_{\LR} = (0.3, \vec{0}_8,0,0,0)$ up to $t = mK/2$.\newline $\gamma_{\LR} = (0.6, \vec{0}_8,0,0,0)$ after $t = mK/2$.\\
			\midrule
			\ref{sec:sim_false} & CE with respect to $(\hat{f}_t(X_t), \tilde{X}_t)$  &
			\eqref{eq:mdl_de} with
			$\theta = (2,1,1,1,\vec{0}_4,1, 0,0)$ and $\delta = \vec{0}$ &
			$\gamma_{\LR} = (0.3,\vec{0}_8,0.1,0,0)$ prior to $t = mK/2$.\newline $\gamma_{\LR} = (0.6, \vec{0}_8,0.2,0,0)$ after $t = mK/2$.\\
			\midrule
			\ref{sec:sim_false} & TC with respect to $\hat{f}_t(X_t)$  &
			\eqref{eq:mdl_tc} with
			$\theta = (2,1,1,1,\vec{0}_4,0,1,0)$ and $\delta = \vec{0}$ &
			$\gamma^{(1)} = (0,\vec{0}_8,0,1,-2)$ at all time points.\newline
			$\gamma^{(2)} = (0.2, \vec{0}_8,0,0,0)$ up to $t = mK/2$.\newline
			$\gamma^{(2)} = (0.4, \vec{0}_8,0,0,0)$ after $t = mK/2$.\\
			\midrule
			\ref{sec:sim_false} & TC with respect to $(\hat{f}_t(X_t), \tilde{X}_t)$  &
			\eqref{eq:mdl_tc} with
			$\theta = (2,1,1,1,\vec{0}_4,1,1,0)$ and $\delta=\vec{0}$ &
			$\gamma^{(1)} = (0,\vec{0}_8,0,1,-2)$ at all time points.\newline
			$\gamma^{(2)} = (0.2, \vec{0}_8,0.3,0,0)$ up to $t = mK/2$.\newline
			$\gamma^{(2)} = (0.4, \vec{0}_8,0.6,0,0)$ after $t = mK/2$.\\
			\midrule
			\ref{sec:sim_retrain} &  Model retrained using ridge-penalized logistic regression or gradient boosted trees & 
			\eqref{eq:mdl_de} with
			$\theta = (2,1,1,\vec{0}_{47},0,0,0)$ and $\delta=\vec{0}$ &
			$\gamma_{\LR} = (0.5, \vec{0}_{50},0,0,0)$\\
			\midrule
			\ref{sec:magnitude} &
			Big and small shifts
			&
			\eqref{eq:mdl_de} with
			$\theta = (2,1,1,1,\vec{0}_4,0, 0,0)$ and
			$\delta = (-1.6,-0.8,-0.8,-0.8,\vec{0}_4,0, 0,0)$ or $\delta = (-1,-0.5,-0.5,-0.5,\vec{0}_4,0, 0,0)$  for the big or small shifts, respectively.
			&
			$\gamma_{\LR} = (0.15, \vec{0}_{8},0,0,0)$
			\\
			\midrule
			\ref{sec:trust}
			&
			Risks either shifted symmetrically or only among those with high-risk.&
			\eqref{eq:mdl_de} with
			$\theta = (2,1,1,1,\vec{0}_4,0, 0,0)$ and $\delta = (-1,-0.5,-0.5,-0.5,\vec{0}_4,0, 0,0)$ or $\delta = (-1,-0.5,-0.5,-0.5,\vec{0}_4,0, 0,-0.75)$ for symmetric- or high-risk shifts, respectively.
			&
			No-trust: $\gamma_{\LR} = (0.01, \vec{0}_{8},0,0,0)$\newline
			Calibrated-trust: $\gamma_{\LR} = (1, \vec{0}_{8},0,0,0)$\newline
			Over-trust: $\gamma_{\LR} = (5, \vec{0}_{8},0,0,0)$\newline
			\\
			\midrule
			\ref{sec:violate}
			&
			Violations of the time-constant selection bias assumption are introduced at time $t'$
			&
			\eqref{eq:mdl_tc} with
			$\theta = (2,1,1,1,\vec{0}_4,0, 0,0)$ and $\delta = (-0.1, -0.02, \vec{0}_6,0, 0,0)$&
			$\gamma^{(1)} = (0,\vec{0}_8,0,1,-1)$ up to $t'$.\newline
			$\gamma^{(1)} = (-0.5,\vec{0}_8,0,1,-1)$ after $t'$.\newline
			$\gamma^{(2)} = (0.8, \vec{0}_8,0,0,0)$ at all time points.
		\end{tabular}
		\caption{
			Model parameters used to generate outcomes and treatment assignments in the simulations.
		}
		\label{table:sim_deets}
	\end{table}
\end{landscape}
\restoregeometry

\section{Na\"ive monitoring of the misclassification rate}

Here we present an additional simulation that compares the false alarm rate of score-based monitoring with a na\"ive CUSUM procedure that monitors the overall misclassification rate to our proposed score-based CUSUM procedure that monitors the conditional distribution of $Y_t(0)|\hat{f}_t(X_t)$.
We assume conditional exchangeability with respect to $\hat{f}_t(X_t)$.
The outcome is generated per \eqref{eq:mdl_de} with $Z_t = X_t$ and $\delta = (2,1,1,1,\vec{0}_4,0)$.
Treatment is assigned using a logistic regression model with input as $\hat{f}_t(X_t)$ with parameter $\gamma_{\LR} = (1,-0.5)$ up to time $t=200$ and $\gamma_{\LR} = (5,-2.5)$ thereafter.
Rather than fitting a risk prediction model, we fit a locked binary classifier for simplicity and classify any observation to be positive if the predicted risk exceeds 0.7.

As shown in Figure~\ref{fig:naive}, the na\"ive CUSUM has a substantially inflated false alarm rate because the overall misclassification rate, without further adjustment, is sensitive to shifts in clinician trust.
On the other hand, the proposed score-based CUSUM procedure controls the false alarm rate at the desired level.
\begin{figure}
	\centering
	\includegraphics[width=0.3\linewidth]{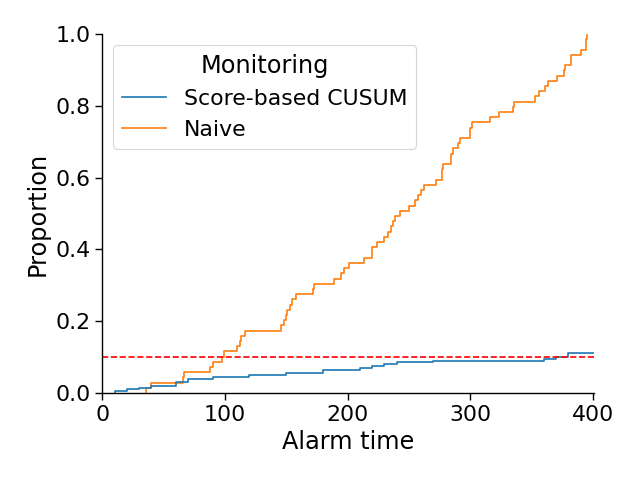}
	\caption{Comparison of na\"ive CUSUM procedure that monitors the unadjusted overall misclassification rate versus score-based CUSUM monitoring of the conditional distribution $Y_t(0)|\hat{f}_t(X_t)$}
	\label{fig:naive}
\end{figure}

\end{document}